\documentclass[11pt]{article}

\usepackage[final]{acl}

\usepackage{times}
\usepackage{latexsym}
\usepackage[T1]{fontenc}
\usepackage[utf8]{inputenc}
\usepackage{microtype}
\usepackage{inconsolata}
\usepackage{graphicx}
\usepackage{float}
\usepackage{adjustbox}
\newcommand{\fiteq}[1]{\adjustbox{max width=0.92\linewidth}{$\displaystyle #1$}}

\usepackage{amsmath}
\usepackage{amssymb}
\usepackage{mathtools}
\usepackage{amsthm}
\usepackage{booktabs}
\usepackage{multirow}
\usepackage{subcaption}
\usepackage{colortbl}
\usepackage{xcolor}
\usepackage{nicefrac}
\usepackage{siunitx}
\usepackage{bbding}
\usepackage{pifont}
\usepackage{algorithm}
\usepackage{algorithmic}

\theoremstyle{plain}
\newtheorem{theorem}{Theorem}[section]
\newtheorem{proposition}[theorem]{Proposition}

\theoremstyle{definition}
\newtheorem{definition}[theorem]{Definition}

\theoremstyle{remark}

\usepackage[capitalize,noabbrev]{cleveref}

\title{FluxBin: Flexible LUT-based Ultra-low-bit LLM Inference by Algorithm-Kernel Synergy}

\author{
 \textbf{Qingyao Yang\textsuperscript{1}},
 \textbf{Runming Yang\textsuperscript{1}},
 \textbf{He Xiao\textsuperscript{1}},
 \textbf{Wendong Xu\textsuperscript{1}},
\\
 \textbf{Junyu Chen\textsuperscript{1}},
 \textbf{Haobo Liu\textsuperscript{1}},
 \textbf{Chenchen Ding\textsuperscript{1}},
 \textbf{Ruihan Hu\textsuperscript{2}},
 \textbf{Yik-Chung Wu\textsuperscript{1}},
 \textbf{Ngai Wong\textsuperscript{1}}
\\
\\
 \textsuperscript{1}The University of Hong Kong,
 \textsuperscript{2}Harbin Institute of Technology
\\
 \small{
   \textbf{Correspondence:} \href{nwong@eee.hku.hk}{nwong@eee.hku.hk}
 }
}

\begin{document}
\maketitle

\begin{abstract}
While binary quantization theoretically promises extreme compression and acceleration for Large Language Models (LLMs), existing research often overlooks the necessity of specialized hardware kernels, thus failing to unleash the full acceleration potential due to persistent reliance on expensive floating-point arithmetic or runtime dequantization overheads. To bridge this gap, we propose FluxBin (\textbf{F}lexible \textbf{L}UT-based \textbf{U}ltra-low-bit e\textbf{X}ecution with \textbf{Bin}ary bases), an algorithm-kernel co-design that synergizes post-training quantization with a highly optimized CUDA kernel. Algorithmically, we introduce Decoupled Row-Column Binary Decomposition to enhance representational capacity while maintaining hardware efficiency, complemented by a Hessian-guided saliency-aware hybrid bases that preserve critical information. At the kernel level, we implement a Lookup Table Building Approach with Scale Fusion to reduce floating-point arithmetic, featuring a Virtual Columnar Mapping that transforms irregular, sparse, and salient matrices into dense execution. Extensive evaluations demonstrate FluxBin achieves up to $5.92\times$ speedup and $10.19\times$ energy savings across diverse model architectures, delivering comparable accuracy to heavily fine-tuned methods.
This effectively enables the deployment of 70B-scale models on one single A100 GPU with a $4\times$ memory reduction. Code is available at \url{https://github.com/nicyyyy/FluxBin}.
\end{abstract}

\section{Introduction}
\label{sec:intro}

While Large Language Models (LLMs) demonstrate unprecedented capabilities~\cite{gpt3, llama2}, their massive scale imposes severe storage and bandwidth constraints~\cite{damf}. Post-Training Quantization (PTQ)~\cite{gptq, awq, leiq} has thus emerged as a standard solution for reducing bit-width without retraining. Among these strategies, {binary quantization}~\cite{billm, xnor-net_2016} pushes compression to the limit, offering maximal theoretical efficiency. 

To mitigate the inherent accuracy degradation of 1-bit representations, recent works like QBB~\cite{QBB}, PTQTP~\cite{ptqtp}, and DB-LLM~\cite{DB-LLM} adopt {multi-binary base quantization}, which approximates weights as combinations of binary matrices to recover representational capacity. However, despite this theoretical promise and improved fidelity, existing approaches still struggle to translate binary representations into practical end-to-end inference speedups.

\textbf{Limitation 1: Failure to exploit multiplication free hardware potential.} While binary representations theoretically allow replacing expensive floating-point operations with efficient bitwise logic, most current binary quantization methods fail to translate this into practical end-to-end acceleration. Methods such as BiLLM~\cite{bi}, ARB-LLM~\cite{ARB-LLM} and HB-LLM~\cite{HBLLM} etc. often prioritize algorithmic fidelity (e.g., simulated quantization) and rely on dequantization to FP16 during inference, leaving the hardware advantages of bitwise computation largely untapped.

\textbf{Limitation 2: Trade-off between saliency granularity and hardware efficiency.} 
Current strategies fail to balance fine-grained preservation with execution efficiency. Unstructured methods~\cite{pb-llm, SpQR, Llm-mq} suffer from irregular memory access and metadata overhead due to element-wise sparsity. Conversely, coarse-grained approaches like {SLiM-LLM}~\cite{slim-llm} promote entire groups to maintain structure, causing {bit redundancy} by unnecessarily preserving non-salient elements at high precision. Thus, a strategy that achieves fine-grained saliency adaptation and hardware efficiency remains absent.

\textbf{Limitation 3: Implementation inefficiencies in low-bit kernels.}
The autoregressive generation phase of LLMs is intrinsically memory-bound, restricting hardware throughput even with INT8 precision~\cite{int8NVIDIA}. Although recent W4A16 quantization methods~\cite{zeng2023glm130, gptq} mitigate memory bottlenecks by compressing weights, they introduce significant runtime overhead due to dynamic dequantization, as weights must be converted to FP16 before matrix multiplication, which does not truly exploit the \textit{addition operation} properties during inference.
Look-Up Tables (LUT)~\cite{lut-gemm, lut-nn,BiQGEMM, UNPU, maleki2023lookupmaigemmincreasing} theoretically offer a multiplication-free alternative by replacing arithmetic with memory lookups. Existing LUT kernels fail to handle the irregular memory access introduced by mixed-precision, and struggle to achieve high performance and low latency in ultra-low bit PTQ settings.

To address the aforementioned limitations simultaneously and bridge the gap between ideal model compression schemes and real-world hardware inference efficiency. We propose {FluxBin}, an algorithm-kernel co-designed framework that works in tandem with flexible and high-performance PTQ method and CUDA kernels. 
\textbf{\textit{At the algorithm level}}, to reconcile binary efficiency with capacity, we propose {Decoupled Row-Column Binary Decomposition} for fine-grained distribution capture. Then, our Hessian-guided {Salient-Aware Hybrid Bases} (mixed  precision) preserve fidelity by selectively allocating additional bases to critical features. Besides, the proposed method is pure PTQ without any fine-tuning~\cite{ft} or distillation~\cite{polino2018model, 8578384}.
\textbf{\textit{At the kernel level}}, our {LUT Building with Scale Fusion (LUT-BSF)} reduces floating-point overhead in LUT building. Complementing this, {Virtual Columnar Mapping (VCM)} converts sparse salient columns into dense storage, enabling unified execution for hybrid branches. Our contributions are:
\begin{itemize}

    \item FluxBin entails a pure PTQ framework, with decoupled row-column binary base and Hessian-guided hybrid base. This maintains the hardware efficiency and competitive accuracy. 

    \item FluxBin designs a highly optimized LUT-based CUDA kernel with LUT-BSF and VCM strategy. This synergy enables high-fidelity hybrid precision and turns binary based PTQ into real inference advantage.
    
    \item FluxBin bridges the algorithm-kernel gap through a custom LUT-based kernel that bypasses de-quantization overhead, converting multiplication-free logic into practical acceleration. It achieves up to \(5.92\times\) speedup and \(10.19\times\) energy saving.
\end{itemize}

\section{Related Work}
\textbf{Binary Quantization.}
Standard binary quantization aims to approximate a full-precision weight matrix \(\mathbf{W} \in \mathbb{R}^{M \times N}\) using a binary matrix \(\mathbf{B} \in \{-1, +1\}^{M \times N}\) scaled by a floating-point factor \(\boldsymbol{\alpha} \in \mathbb{R}\)~\cite{BiBench}. The optimization objective is to minimize the reconstruction error under the Frobenius norm:
\begin{equation}
\small
\min_{\boldsymbol{\alpha}, \mathbf{B}} \| \mathbf{W} - \boldsymbol{\alpha} \mathbf{B} \|_F^2
\label{equa:basic-binary}
\end{equation}
The optimal solution to Eq.~\eqref{equa:basic-binary} typically yields \(\mathbf{B} = \text{sign}(\mathbf{W})\) and \(\boldsymbol{\alpha} = \frac{1}{MN} \|\mathbf{W}\|_1\), as derived in prior works like XNOR-Net~\cite{xnor-net_2016}. Regarding the representative works, ARB-LLM~\cite{ARB-LLM} introduces column-wise grouping
and alternating refined binarization. HB-LLM~\cite{HBLLM} applies distinct binarization strategies to frequency components separated by wavelet decomposition.

\textbf{Binary Base Approximation.}
Since a single binary basis (1-bit) has limited representation capacity, high order binary base quantization methods approximate \(\mathbf{W}\) using \(K\) binary bases:
\begin{equation}
\small
\mathbf{W} \approx \sum_{k=1}^{K} \boldsymbol{\alpha}_k \mathbf{B}^{(k)}
\label{equa:multi-base}
\end{equation}

To minimize this quantization error, {LQ-Nets}~\cite{LQ-Nets} proposes to jointly train the quantization basis and the binary encodings. This allows for an arbitrary basis that better fits the data distribution. Approaches such as {QBB}~\cite{QBB} solve this via a greedy residual strategy. Let \(\mathbf{R}^{(0)} = \mathbf{W}\). At the \(k\)-th step, parameters are optimized to approximate the current residual, updating \(\mathbf{R}^{(k)} = \mathbf{R}^{(k-1)} - \alpha_k \mathbf{B}^{(k)}\). DB-LLM~\cite{DB-LLM} approximates weights by primary and secondary residual binary base to refine the quantization errors. %

\begin{figure}[ht]
    \centering
    \includegraphics[width=0.8\linewidth]{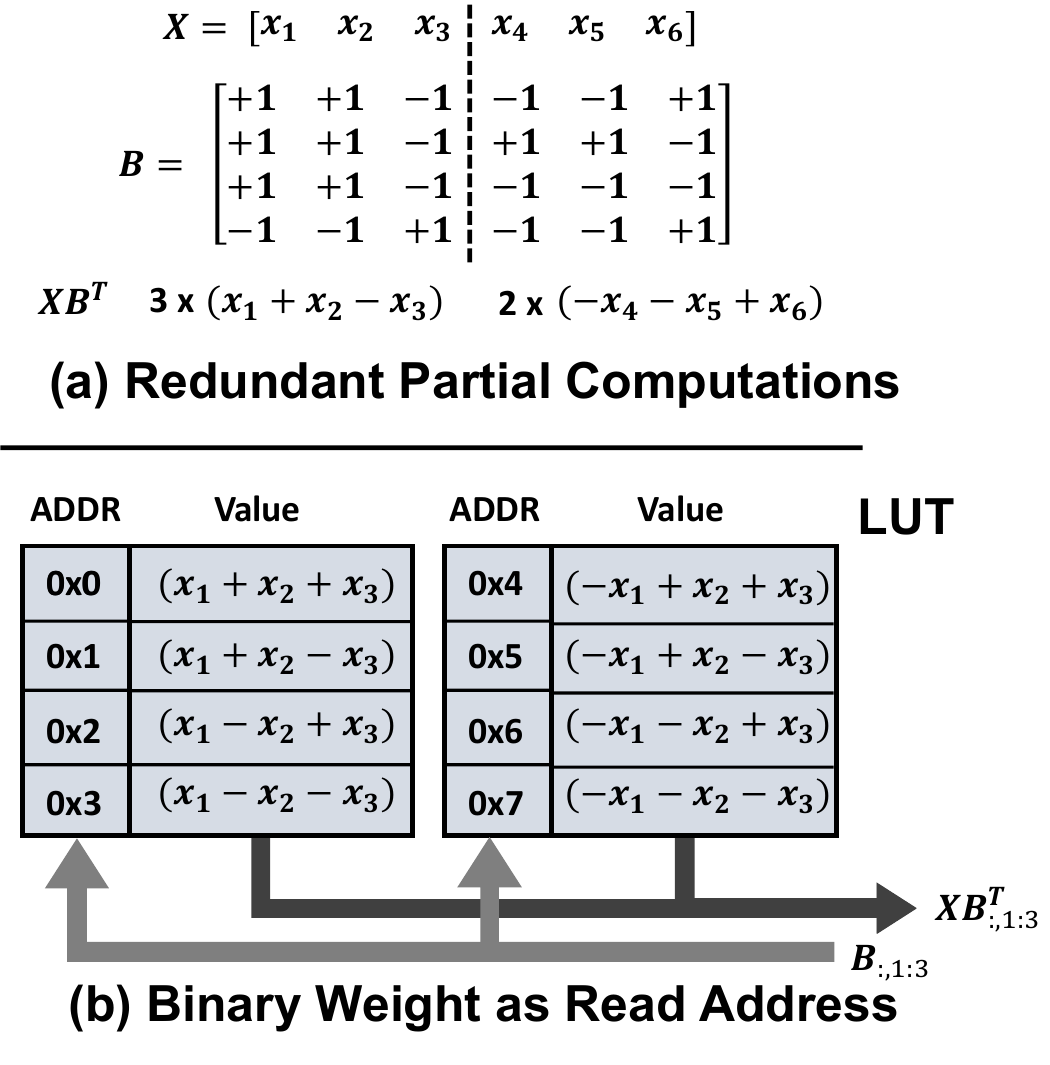}
    \vspace{-0.5em}
    \caption{Illustration of LUT-based Multiplication.}
    \label{fig:lut-principle}
\end{figure}

\begin{figure*}[t]
    \centering
    \includegraphics[width=0.6\linewidth]{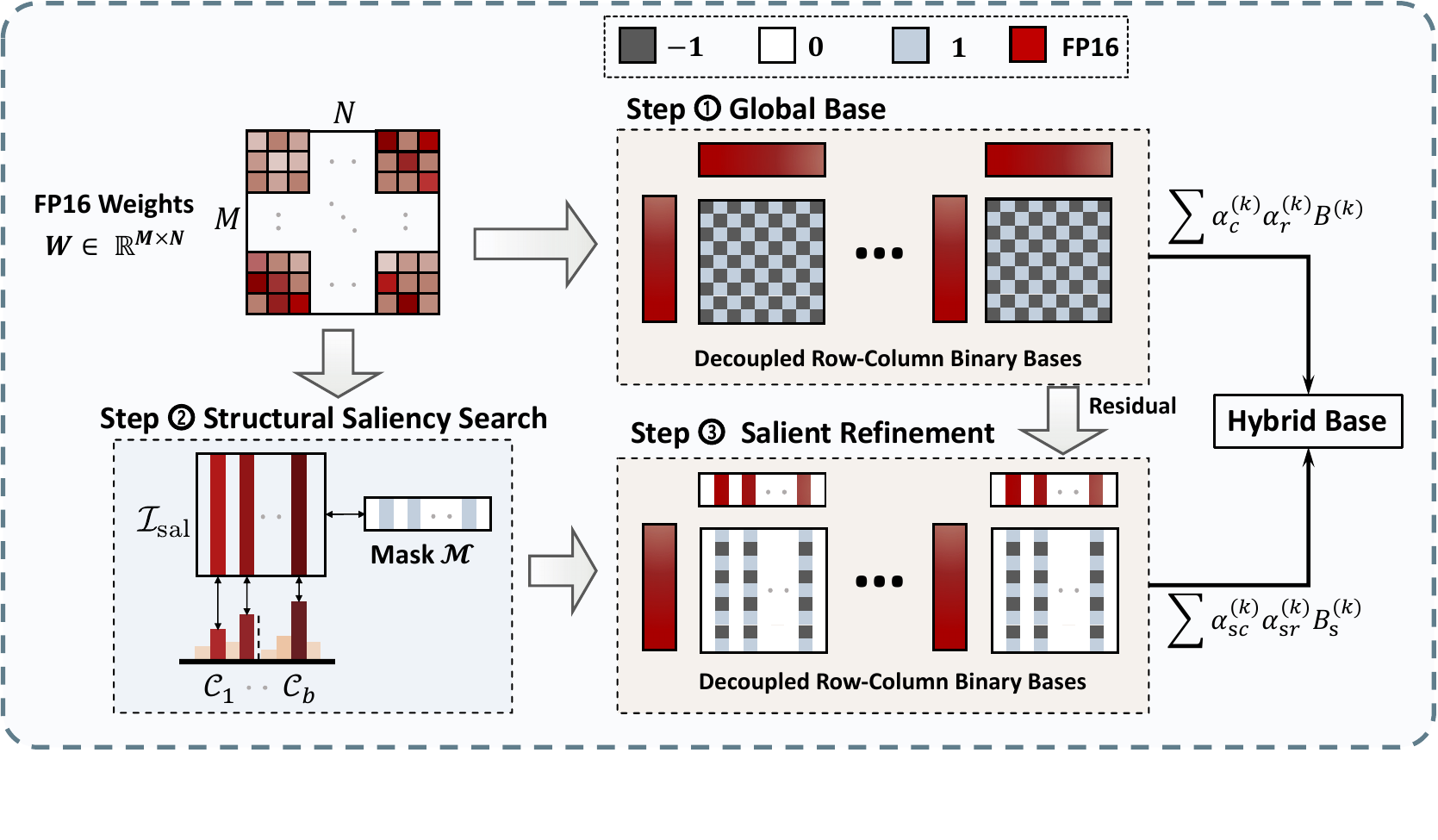}
    \vspace{-1.0em}
    \caption{The framework of FluxBin Binary Bases Quantization.}
    \label{fig:quant-method}
\end{figure*}

\textbf{LUT-based Multiplication.}
Standard matrix-vector multiplication (MVM) with binary weights \(\mathbf{B} \in \{-1, +1\}^{M \times N}\) often incurs significant computational redundancy. As illustrated in Figure~\ref{fig:lut-principle}(a), identical weight sub-patterns (e.g., \([+1, +1, -1]\)) frequently repeat across different rows. Consequently, the dot product for the same input sub-vector is calculated multiple times redundantly.

To address this, {LUT-GEMM}~\cite{lut-gemm} proposes replacing arithmetic operations with efficient memory lookups. The input dimension is partitioned into sub-vectors of size \(\mu\) (e.g., \(\mu=3\) in the figure). For each input sub-vector \(\mathbf{X}_{sub}\), all \(2^\mu\) possible linear combinations are pre-computed and stored in a LUT. As shown in Figure~\ref{fig:lut-principle}(b), during inference, the \(\mu\)-bit binary weight pattern serves directly as the memory address to fetch the pre-computed partial sum, which is much faster than carrying out the original calculations~\cite{9478922, BiQGEMM}. Furthermore, recent work LUT Tensor Core~\cite{LUT-tensor-core} introduced an instruction set, with an elaborated precompute and tiling strategy, to build LUT-based accelerators. While LUT-GEMM provides an efficient kernel for uniform quantization, it does not provide PTQ algorithm adaptation, mixed-precision configurations, or end-to-end inference pipelines.

\section{Methodology}

\subsection{FluxBin Quantization}
We present the mathematical formulation of FluxBin quantization, visualized in Figure~\ref{fig:quant-method}, and the pseudo-code of this quantization method is described in Sec.~\ref{sec:algo_desc}.

\textbf{Decoupled Row-Column Binary Bases.}
We approximate dense weights \( \mathbf{W} \in \mathbb{R}^{M \times N} \) via linear combinations of binary bases. Distinct from prior row-scaling-only methods, we incorporate column-wise scaling to capture anisotropic weight distributions.

\begin{definition}[Row-Column Binary Decomposition]
\label{def:rc-decomp}
For decomposition order \(k\), the approximation \(\hat{\mathbf{W}}^{(k)}\) is defined as the Hadamard product of scaling vectors and a binary base:
\begin{equation}
\small
\hat{\mathbf{W}}^{(k)} = (\boldsymbol{\alpha}_{r}^{(k)} \boldsymbol{\alpha}_{c}^{(k)\top}) \odot \mathbf{B}^{(k)}
\label{equa:rc-decompoes}
\end{equation}
where \( \mathbf{B}^{(k)} \in \{-1, +1\}^{M \times N} \) is the binary base, and \( \boldsymbol{\alpha}_{r}^{(k)} \in \mathbb{R}^{M} \), \( \boldsymbol{\alpha}_{c}^{(k)} \in \mathbb{R}^{N} \) denote the row and column scaling vectors.
\end{definition}

The final quantized matrix sums over all orders: \( \hat{\mathbf{W}} = \sum_{k=1}^{K_b} \hat{\mathbf{W}}^{(k)}\). Consistent with \cite{gptq, billm}, we assign one row scale \(\boldsymbol{\alpha}_{r}^{(k)}\) per weight group.

\textbf{Parameter Optimization.}
We optimize parameters \(\{\boldsymbol{\alpha}_r, \boldsymbol{\alpha}_c, \mathbf{B}\}\) by minimizing the reconstruction error:
\begin{equation}
\small
\mathcal{L}=\left\|\mathbf{W}-\left(\boldsymbol{\alpha}_r \boldsymbol{\alpha}_c^\top\right) \odot \mathbf{B}\right\|_F^2
\label{equa:loss}
\end{equation}
Employing alternating minimization, fixing \(\mathbf{B}\) and one scaling vector reduces the problem to weighted least squares.

\begin{proposition}[Optimal Scaling Factors]
Given fixed \(\mathbf{B}\) and one scaling vector, the closed-form solutions minimizing \(\mathcal{L}\) are:
\vspace{-0.5em}
\begin{equation}
\small
\begin{aligned}
{\boldsymbol\alpha}^{(k)}_{r, i}&=\frac{\sum_{j}\left(\boldsymbol{\alpha}^{(k)}_{c, j} \mathbf{B}^{(k)}_{ij}\right) \mathbf{W}_{ij}}{\sum_{j}\left({\boldsymbol\alpha}^{(k)}_{c, j} \mathbf{B}^{(k)}_{ij}\right)^2+\epsilon}, \\
{\boldsymbol\alpha}^{(k)}_{c, j}&=\frac{\sum_{i}\left({\boldsymbol\alpha}^{(k)}_{r, i} \mathbf{B}^{(k)}_{ij}\right) \mathbf{W}_{ij}}{\sum_{i}\left({\boldsymbol\alpha}^{(k)}_{r, i} \mathbf{B}^{(k)}_{ij}\right)^2+\epsilon}
\end{aligned}
\label{equa:row-scale-update}
\end{equation}
where $\epsilon$ is a stability constant.
\end{proposition}

For \(K_b \geq 2\), we perform joint optimization rather than greedy sequential extraction. Defining the magnitude tensor \(\mathbf{M}^{(k)}=\boldsymbol{\alpha}_{r}^{(k)} \boldsymbol{\alpha}_{c}^{(k)\top}\), we solve for the optimal binary signs:

\begin{proposition}[Optimal Binary Base Assignment]
The binary combination \(\mathbf{b} = [b_1, \dots, b_{K_b}]^\top\) minimizing the local error at \((i,j)\) is:
\begin{equation}
\small
\mathbf{B}_{i j}^{(1: K_b)}=\underset{\mathbf{b} \in\{-1,+1\}^{K_b}}{\operatorname{argmin}}\left|\mathbf{W}_{i j}-\sum_{k=1}^{K_b} \mathbf{M}_{i j}^{(k)} \cdot b_k\right|
\label{equa:B-update}
\end{equation}
\end{proposition}
Implementation-wise, we exhaustively evaluate all \(2^{K_b}\) sign combinations to ensure optimality.

\textbf{Structural Saliency Search.}
Recognizing that weight elements contribute unequally to performance, we adopt a hybrid base strategy guided by structural saliency. Leveraging the Optimal Brain Quantization (OBQ) framework~\cite{OBQ}, we approximate weight sensitivity via the inverse Hessian \(\mathbf{H}^{-1}\), and this can be efficiently calculated through Cholesky decomposition~\cite{6710599}. Given the column-wise clustering of sensitive weights in LLMs~\cite{billm}, we aggregate element-wise errors into a cumulative column score, following previous works~\cite{slim-llm, SparseGPT}.

\begin{definition}[Hessian-based Column Saliency]
Following SLiM-LLM~\cite{slim-llm}, the column-wise saliency score \(\mathbf{S}_{j}\) is defined as:
\begin{equation}
\small
\mathbf{S}_{j}=\sum\limits_{i=1}^{M}\frac{\mathbf{W}_{ij}^2}{\left[\mathbf{H}^{-1}\right]_{j j}^2}
\label{equa:saliency-score-column}
\end{equation}
\end{definition}

Columns with higher \(\mathbf{S}_{j}\) usually contain outliers. To ensure load balancing, we partition columns into groups of size \(g\). Within each group \(\mathcal{C}_b\), we select a subset \(\mathcal{I}^{(b)}_{\text{sal}}\) of size \(s\) (the number of salient columns per group) maximizing total saliency:
\begin{equation}
\small
\mathcal{I}^{(b)}_{\text{sal}} = \mathop{\arg\max}\limits_{\mathcal{I} \subset \mathcal{C}_b, |\mathcal{I}|=s} \sum_{j \in \mathcal{I}} \mathbf{S}_{j}
\label{equa:saliency-column-idx}
\end{equation}
The aggregated set \(\mathcal{I}_{\text{sal}} = \bigcup_b \mathcal{I}^{(b)}_{\text{sal}}\) (with \(|\mathcal{I}_{\text{sal}}| = N_{\text{sal}} = (N/g)\,s\) salient columns in total) defines the binary mask \(\mathcal{M}\) (\(\mathcal{M}_{j}=1\) if \(j \in \mathcal{I}_{\text{sal}}\), else 0).

\textbf{Salient-Aware Hybrid Bases.}
The final quantized weight \( \hat{\mathbf{W}} \) combines a Global Base Approximation and a Salient Refinement (Figure~\ref{fig:quant-method}(b)):
\begin{equation}
\small
\hat{\mathbf{W}}=\underbrace{\sum_{k=1}^{K_{b}}\mathcal{Q}(\mathbf{W};\mathbf{\Theta}_{k})}_{\text{Global Base}}+\underbrace{\sum_{k=1}^{K_s}\mathcal{Q}(\mathbf{R}_\text{base};\mathbf{\Phi}_{k}) \odot \mathcal{M}}_{\text{Salient Refinement}}
\label{equa:hybrid-base}
\end{equation}
Here, \(\mathcal{Q}(\cdot)\) denotes the quantization function (Def.~\ref{def:rc-decomp}), with parameter sets \(\mathbf{\Theta}_{k}=\{\mathbf{B}^{(k)},\boldsymbol{\alpha}_{r}^{(k)},\boldsymbol{\alpha}_{c}^{(k)}\}\) for the global branch and \(\mathbf{\Phi}_{k}=\{\mathbf{B}_{s}^{(k)},\boldsymbol{\alpha}_{sr}^{(k)},\boldsymbol{\alpha}_{sc}^{(k)}\}\) for the salient branch. \(\mathbf{R}_\text{base}\) represents the residual of the global approximation: \(\mathbf{R}_\text{base} = \mathbf{W} - \sum_{k}\mathcal{Q}(\mathbf{W};\mathbf{\Theta}_{k})\).

\subsection{FluxBin LUT-based Kernel}
Figure \ref{fig:kernel-overview} illustrates the FluxBin LUT-based GEMV kernel architecture, comprising three stages: \textbf{(i)} LUT Building with Scale Fusion to remove column scaling \(\boldsymbol{\alpha}_c\) from the critical path; \textbf{(ii)} LUT reading for Multiplication-free GEMV, utilizing bitwise operations
for high throughput; and \textbf{(iii)} Reduction with Row-Scale, applying row scaling \(\boldsymbol{\alpha}_r\) and performing reduction before writing back to global memory. Crucially, our Virtual Columnar Mapping (VCM) allows unified kernel execution. 

\begin{figure*}[t]
    \centering
    \begin{subfigure}[t]{0.6\linewidth}
        \centering
        \includegraphics[width=\linewidth]{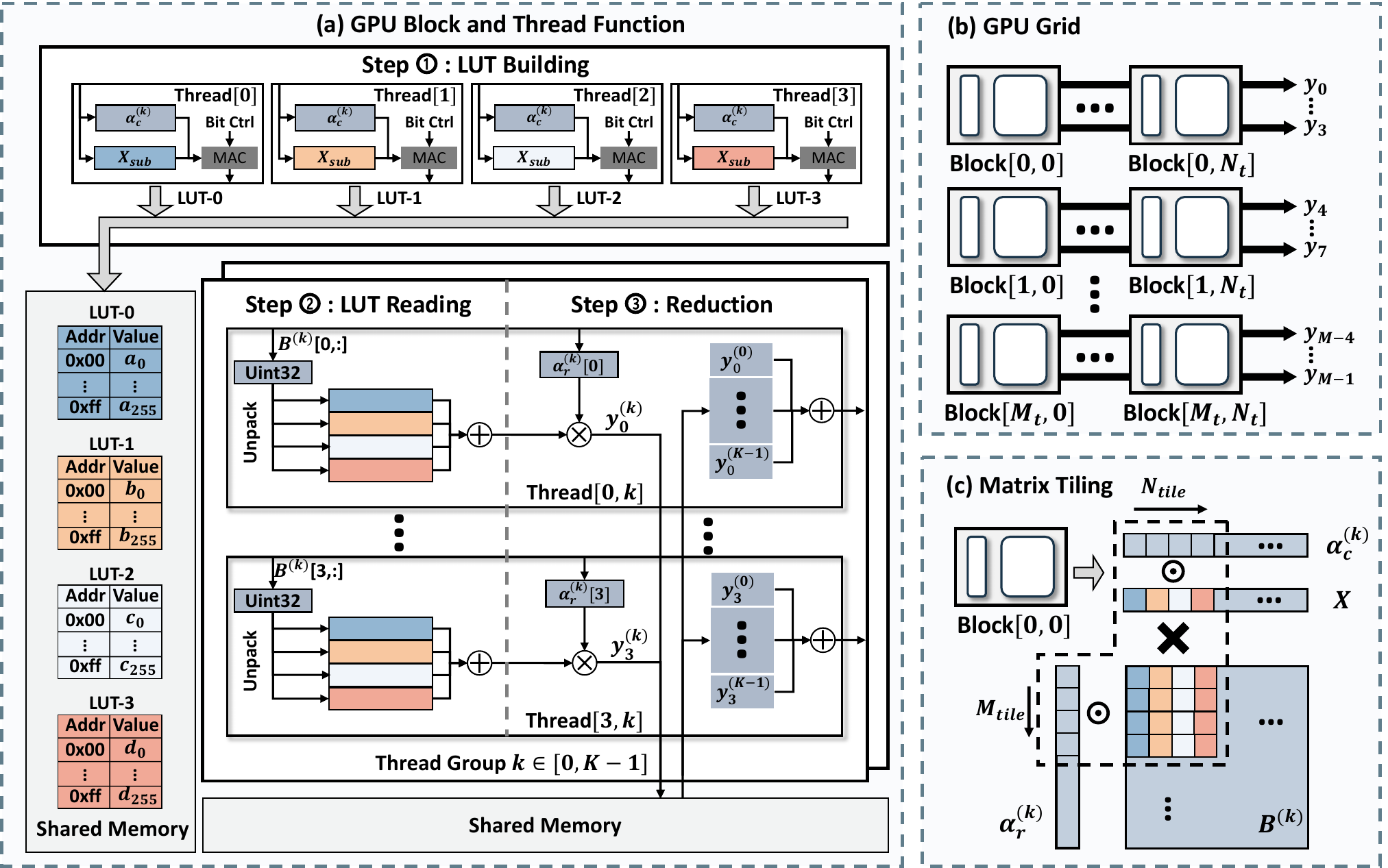}
        \caption{Overview of the FluxBin LUT-based GEMV Kernel.}
        \label{fig:kernel-overview}
    \end{subfigure}
    \hfill
    \begin{subfigure}[t]{0.38\linewidth}
        \centering
        \includegraphics[width=\linewidth]{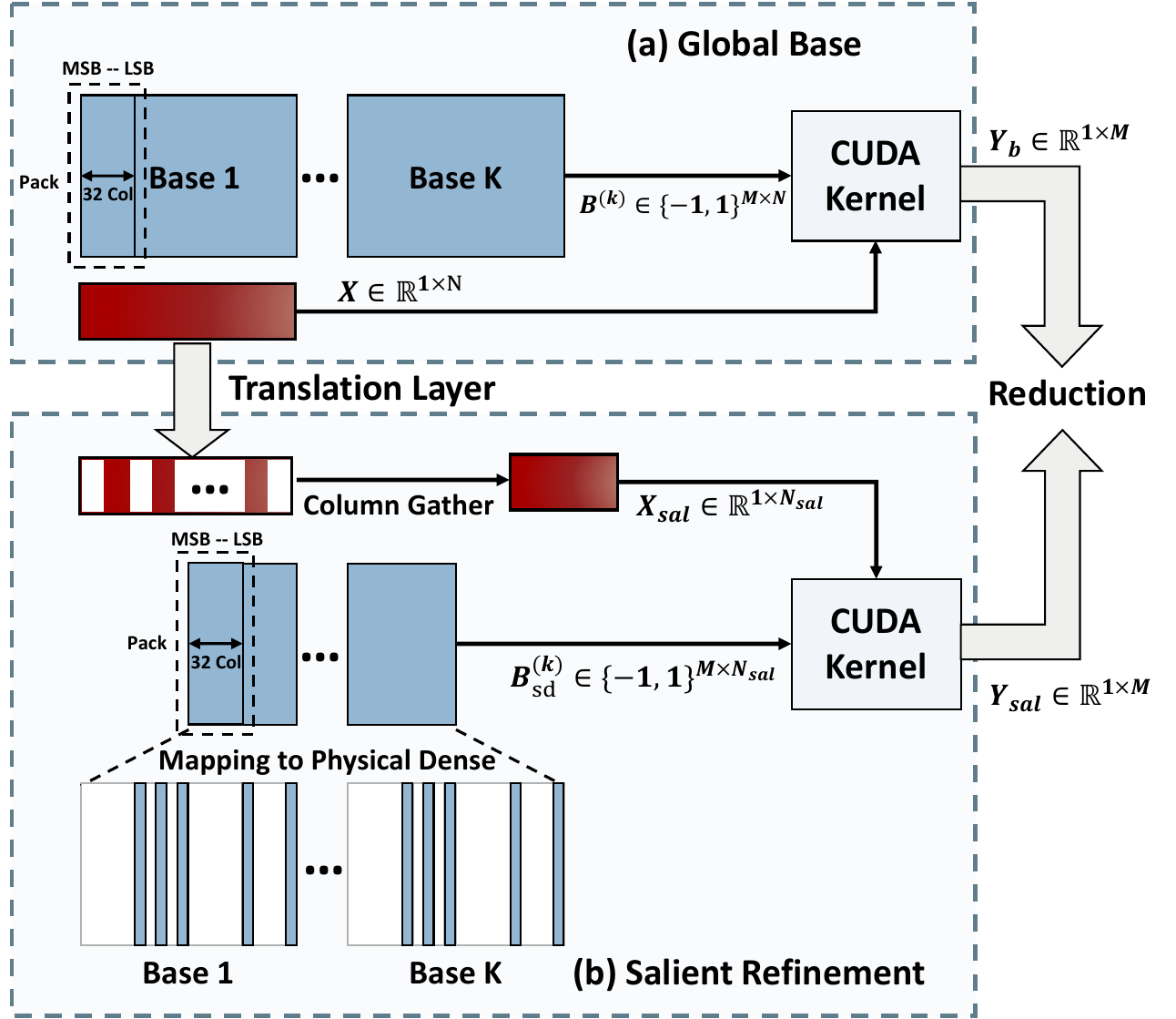}
        \caption{Virtual Columnar Mapping.}
        \label{fig:salient-gather}
    \end{subfigure}
    \vspace{-0.5em}
    \caption{FluxBin kernel architecture and VCM strategy.}
    \label{fig:kernel-vcm}
\end{figure*}

Based on Eqs.~\eqref{equa:rc-decompoes} and \eqref{equa:hybrid-base}, the GEMV is:
\begin{equation}
\small
\begin{aligned}
\mathbf{X} \hat{\mathbf{W}}^T ={}& \underbrace{\mathbf{X} \!\left(\textstyle\sum_{k=1}^{K_b} \mathbf{B}^{(k)} \!\odot\! (\boldsymbol{\alpha}_{r}^{(k)} \boldsymbol{\alpha}_{c}^{(k)\top})\right)^{\!T}}_{\text{Global Base}} \\
&+ \underbrace{\mathbf{X} \!\left(\textstyle\sum_{k=1}^{K_s} (\mathbf{B}_{{s}}^{(k)} \!\odot\! \mathcal{M}) \!\odot\! (\boldsymbol{\alpha}_{sr}^{(k)} \boldsymbol{\alpha}_{sc}^{(k)\top}) \right)^{\!T}}_{\text{Salient Refinement}}
\end{aligned}
\label{equa:gmev-fluxbin}
\end{equation}

Since \(\boldsymbol{\alpha}_c^{(k)}\) and \(\boldsymbol{\alpha}_r^{(k)}\) are diagonal per-column and per-row factors, the Global Base GEMV \(\mathbf{Y}_\text{b}\) factorizes so that the column scale folds into the input and the row scale is applied after the binary matvec:
\begin{equation}
\small
\mathbf{Y}_\text{b} = \sum\limits_{k=1}^{K_b} \left( (\mathbf{X} \odot \boldsymbol{\alpha}_{c}^{(k)\top}) \, \mathbf{B}^{(k)T} \right) \odot \boldsymbol{\alpha}_{r}^{(k)\top}
\label{equa:global-base-GEMV}
\end{equation}

\textbf{Virtual Columnar Mapping (VCM) for Salient Phase.} 
The GEMV computation of salient refinement is restricted to sparse columns defined by \(\mathcal{I}_{\text{sal}}\). Since conventional sparse implementations suffer from irregular memory access, we introduce VCM, inspired by OS memory virtualization. As shown in Figure~\ref{fig:salient-gather}, we map the logical sparse bases to physical dense matrix, and maintain a translation layer (via \(\mathcal{I}_{\text{sal}}\)) to gather activation and scaling factor at runtime.

\begin{definition}[VCM Projection]
\label{def:vcm}
VCM decouples logical sparsity from physical storage. Salient columns are remapped into a contiguous dense matrix \(\mathbf{B}^{(k)}_{sd} \in \{-1, 1\}^{M \times N_{\text{sal}}}\):
\begin{equation}
\small
\mathbf{B}^{(k)}_{sd} = [\mathbf{B}^{(k)}_{s; :, j} \mid j \in \mathcal{I}_{\text{sal}} ]
\label{equa:salient-dense-extract}
\end{equation}
\end{definition}

Prior to GEMV, we gather input activations into \(\mathbf{X}_{\text{sal}}\) and extract scales \(\boldsymbol{\alpha}_{scd}^{(k)}\):
\begin{equation}
\small
\begin{aligned}
\mathbf{X}_{\text{sal}} &= [\mathbf{X}_{:, j} \mid j \in \mathcal{I}_{\text{sal}} ], \\
\boldsymbol{\alpha}_{scd}^{(k)} &= [\boldsymbol{\alpha}_{sc}^{(k)}[j] \mid j \in \mathcal{I}_{\text{sal}} ]
\end{aligned}
\label{equa:salient-column-gather}
\end{equation}
This transforms the logically sparse operation into a physically dense GEMV, enabling kernel reuse for Global Base and Salient Refinement.

\textbf{Stage 1: LUT Building with Scale Fusion (BSF).} 
We partition inputs into sub-vectors \(\mathbf{X}_{\text{sub}}\) of size \(\mu\)~\cite{lut-gemm}. Exploiting the column-wise scaling of \(\boldsymbol{\alpha}_c\) that shared by all rows, we fuse it directly into the LUT, computing entries using \(\mathbf{X} \odot \boldsymbol{\alpha}_c\).

\begin{definition}[LUT-BSF]
For an integer index \(p \in [0, 2^{\mu}-1]\), \(b_n(p)\in \{0,1\}\) denote the $n$-th bit of $p$, the fused LUT entry is:
\begin{equation}
\small
\operatorname{{LUT}_\text{con}}[p] = \sum_{n=1}^{\mu} (\mathbf{X}_{\text{sub}}[n] \cdot \boldsymbol{\alpha}_{c, sub}[n]) \cdot (2 b_n(p) - 1)
\label{equa:LUT-construction}
\end{equation}
\end{definition}
Here, \((2 b_n(p) - 1)\) maps bits to \(\{-1, +1\}\). This strategy decouples \(\boldsymbol{\alpha}_{c}\) from the main GEMV loop, effectively eliminating related floating-point operations.

\begin{figure*}[t]
    \centering
    \includegraphics[width=0.7\linewidth]{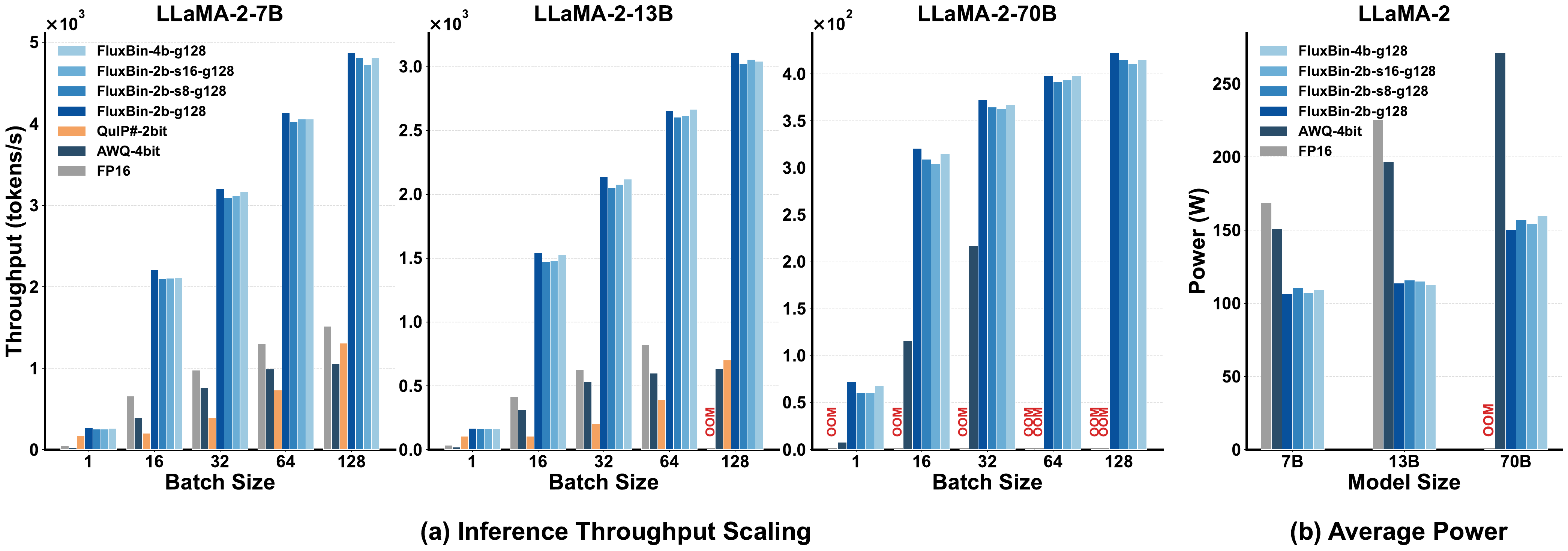}
    \vspace{-0.3em}
    \includegraphics[width=0.7\linewidth]{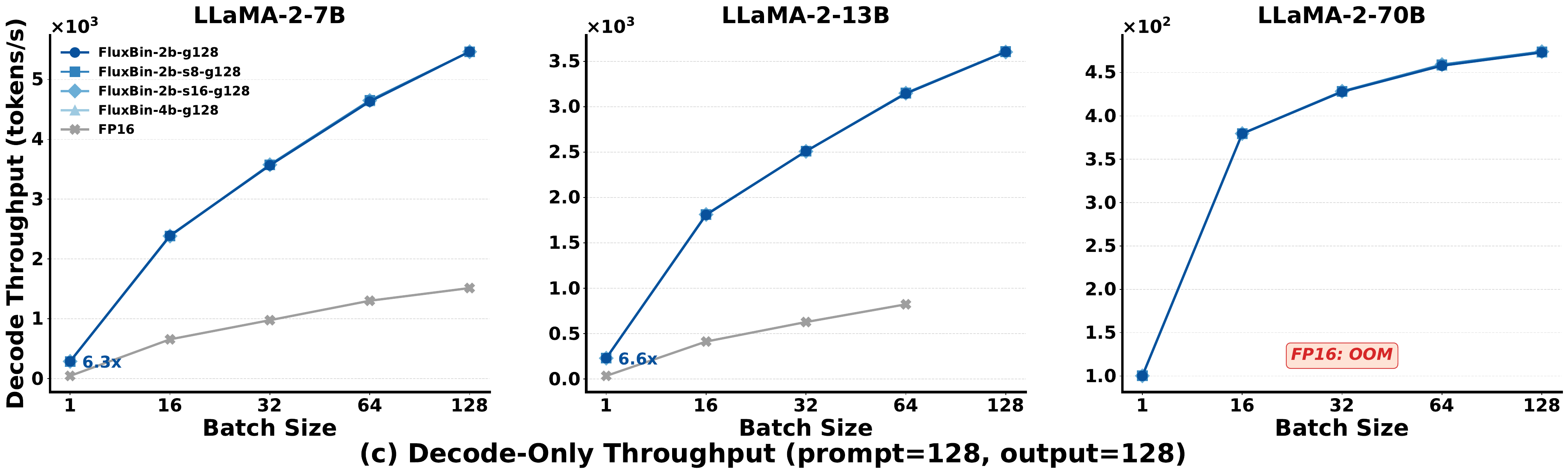}
    \vspace{-0.5em}
    \caption{(a) E2E throughput scaling. (b) Average power. (c) Decode-only throughput.}
    \label{fig:throughput-power-llama2}
\end{figure*}

\textbf{Stage 2\&3: LUT Reading and Reduction.} Instead of unpacking weights to Float16, the kernel directly streams compressed Uint32 weights. For each sub-block of size \(\mu\), we extract pattern indices \(p\) via bitwise ops to retrieve pre-computed partial sums from LUT. The results are
accumulated modulated by the row-wise scaling factor in registers before being written to global memory.
\begin{equation}
\small
\begin{aligned}
\mathbf{Y}_{\text{b}} &= \sum_{k=1}^{K_b} \left( \operatorname{LUT_{\text{read}}} [\mathbf{B}^{(k)}] \odot \boldsymbol{\alpha}_{r}^{(k)\top} \right), \\
\mathbf{Y}_{\text{sal}} &= \sum_{k=1}^{K_s} \left( \operatorname{LUT_{\text{read}}} [\mathbf{B}_{sd}^{(k)}] \odot \boldsymbol{\alpha}_{sr}^{(k)\top} \right)
\end{aligned}
\label{equa:LUT-reading}
\end{equation}
\textbf{Matrix Tiling.}
As shown in Figure~\ref{fig:kernel-overview}~(c), \(\mathbf{B}\) is partitioned into \(M_{\text{tile}} \times N_{\text{tile}}\) tiles mapped to thread blocks. Within each tile, input sub-vectors and \(\boldsymbol{\alpha}_{c, sub}\) are fused to build LUTs in shared memory, queried by all threads in the block.

\section{Experiment}
\label{sec:experiment}

\noindent\textbf{Experimental Setup.} All evaluations are conducted on a single {NVIDIA A100 (80GB) GPU}. For end-to-end throughput, we customize the Hugging Face {generate} API~\cite{wolf2020transformers} with {CUDA graphs} to eliminate Python scheduling and kernel launch overhead: both prefill and decode kernels are captured and replayed from a pre-recorded graph with a static KV cache, ensuring that the measured latency reflects actual kernel execution time rather than framework overhead.

Regarding FluxBin configurations, we adopt the notation {Kb-sS-gG}, where the leading number K denotes the base count (the literal ``b'' stands for ``bases''), shared by the global and salient-refinement branches (we use the same number of bases for both). s denotes salient columns per group, and g is group size for both saliency search and row-wise scale. We primarily evaluate the mixed-precision variants {2b-s8-g128} and {2b-s16-g128}. Additionally, {2b-g128} and {4b-g128} (pure global base) are included to provide a comprehensive analysis.

\definecolor{rank1}{RGB}{210, 210, 210} %
\definecolor{rank2}{RGB}{230, 230, 230} %
\definecolor{rank3}{RGB}{248, 248, 248} %

\newcommand{\best}[1]{\cellcolor{rank1}#1}
\newcommand{\secondbest}[1]{\cellcolor{rank2}#1}
\newcommand{\third}[1]{\cellcolor{rank3}#1}

\begin{table*}[!htbp]
\centering
\setlength{\abovecaptionskip}{2pt}
\setlength{\belowcaptionskip}{0pt}
\caption{\small E2E speed and zero-shot accuracy on LLaMA-2. $^{*}$FT/Dist used. $^{\dagger}$Avg bit incl.\ scales (Sec.~\ref{sec:appendix-storage}). $^{\ddagger}$HS=HellaSwag, WG=Winogrande. \colorbox{rank1}{Best}/\colorbox{rank2}{2nd}/\colorbox{rank3}{3rd}.}
\label{tab:llama2-task}
\setlength{\tabcolsep}{3pt}
\renewcommand{\arraystretch}{0.9}
\tiny
\begin{tabular}{l l c c c c c c c c c c c}
\toprule
\textbf{LLaMA-2} & \textbf{Method} & \textbf{FT/Dist}$^{*}$ & \textbf{W bit}$^{\dagger}$ & \textbf{Speed (tok/s)} $\uparrow$ &
\textbf{Wiki} $\downarrow$ & \textbf{PIQA} $\uparrow$ & \textbf{ARC-e} $\uparrow$ &
\textbf{ARC-c} $\uparrow$ & \textbf{BoolQ} $\uparrow$ &
\textbf{HS}$^{\ddagger}$ $\uparrow$ & \textbf{WG}$^{\ddagger}$ $\uparrow$ & \textbf{Avg.} $\uparrow$ \\
\midrule
\multirow{8}{*}{7B}
& FP16                     &  & 16 & 42.94 & 5.47 & 78.13 & 75.46 & 43.00 & 79.36 & 57.08 & 69.38 & 67.07 \\
& GPTQ	&\XSolidBrush	&3.125	&--	&8.42	&73.50	&63.43	&30.29	&67.00	&\third{48.57}	&\best{64.96}	&57.96 \\
& OmniQuant                & \Checkmark & 3    & 83.40 & \secondbest{6.62} & \third{74.70} & 64.26 & \best{35.92} & 66.29 & \best{51.73} & 63.45 & \secondbest{59.39} \\
& SLiM-LLM              & \Checkmark & 2    &73.70  & 16.38 & 63.43 & 46.42 & 23.63 & 61.71 & 33.22 & 55.64 & 47.34 \\
& AQLM                     & \Checkmark & 2.02 & 98.50 & \best{6.59} & 72.69 & \third{65.11} & \third{32.85} & 69.24 & \secondbest{49.34} & 62.59 & 58.64 \\
& QuIP\#                   & \Checkmark & 2.02 & \third{131.97} & 8.22 & \secondbest{75.10} & 64.60 & \secondbest{34.60} & \best{74.31} & 42.94 & \secondbest{64.90} & \best{59.41} \\
& FluxBin-2b-s8-g128       & \XSolidBrush & 2.63 & \best{254.39} & 8.76 & 72.52 & \best{65.45} & 30.97 & \third{69.88} & 45.64 & 64.09 & 58.09 \\
& FluxBin-2b-s16-g128      & \XSolidBrush & 2.75 & \secondbest{250.87} & \third{8.20} & \best{75.50} & \secondbest{65.40} & 32.42 & \secondbest{70.80} & 45.97 & \third{64.80} & \third{59.15} \\
\midrule
\multirow{8}{*}{13B}
& FP16                     &  & 16 & 33.14 & 4.88 & 79.49 & 78.87 & 47.53 & 82.20 & 60.19 & 72.69 & 70.16 \\
& GPTQ	&\XSolidBrush	&3.125	&--	&6.40	&75.95	&\secondbest{73.23}	&\third{41.89}	&73.64	&54.15	&67.56	&\third{64.40} \\
& OmniQuant                & \Checkmark & 3    & 57.60 & \best{5.58} & \best{77.69} & 69.02 & \secondbest{41.97} & 69.02 & \best{56.79} & 65.09 & 63.26 \\
& SLiM-LLM              & \Checkmark & 2    &  61.20 & 9.41 & 64.96 & 54.67 & 24.82 & 65.22 & 38.43 & 55.48 & 50.60 \\
& AQLM                     & \Checkmark & 2    & 32.90 & \secondbest{5.60} & \third{76.71} & \best{78.79} & 40.36 & 75.81 & \third{54.52} & 64.80 & \secondbest{65.17} \\
& QuIP\#                   & \Checkmark & 2.01 & \third{84.16} & \third{6.06} & \secondbest{77.30} & 69.30 & \best{42.92} & \best{79.24} & \secondbest{55.46} & \third{67.70} & \best{65.32} \\
& FluxBin-2b-s8-g128       & \XSolidBrush & 2.63 & \best{162.94} & 7.43 & 75.46 & 71.80 & 36.52 & \secondbest{79.11} & 49.36 & \secondbest{69.53} & 63.63 \\
& FluxBin-2b-s16-g128      & \XSolidBrush & 2.75 & \secondbest{162.73} & 7.16 & 75.46 & \third{72.81} & 38.05 & \third{76.91} & 50.53 & \best{69.96} & 63.95 \\
\midrule
\multirow{7}{*}{70B}
& FP16                     &  & 16 & OOM   & 3.12 & 81.50 & 82.70 & 54.10 & 85.17 & 65.33 & 80.43 & 74.87 \\
& GPTQ	&\XSolidBrush	&3.125	&--	&4.96	&79.27	&78.41	&48.12	&79.51	&59.10	&74.66	&69.85 \\
& OmniQuant                & \Checkmark & 3    & --    & \best{3.92} & \third{79.70} & 75.58 & 47.52 & 66.48 & \secondbest{61.92} & 73.71 & 67.49 \\
& AQLM                     & \Checkmark & 2.07 & 6.82  & \secondbest{3.94} & \best{80.30} & \third{78.79} & \best{51.19} & \secondbest{82.29} & \best{62.23} & 75.77 & \best{71.76} \\
& QuIP\#                   & \Checkmark & 2.01 & \third{18.64} & \third{4.21} & \best{80.30} & 77.30 & 48.70 & \best{82.91} & \secondbest{61.92} & \third{75.90} & \secondbest{71.17} \\
& FluxBin-2b-s8-g128       & \XSolidBrush & 2.63 & \best{60.67} & 4.90 & 78.64 & \secondbest{80.22} & \secondbest{49.83} & \third{81.04} & 59.16 & \best{76.80} & \third{70.95} \\
& FluxBin-2b-s16-g128      & \XSolidBrush & 2.75 & \secondbest{60.60} & 4.84 & 79.49 & \best{80.26} & \third{49.06} & 79.54 & 59.34 & \secondbest{76.32} & 70.67 \\
\bottomrule
\end{tabular}

\vspace{4pt}
\scriptsize
\raggedright

\end{table*}

\subsection{End-to-end Analysis}

We compare FluxBin against the FP16 baseline and SOTA quantization frameworks, including AutoAWQ (4-bit) and QuIP\# (2-bit).

\textbf{Throughput and Scalability.} As shown in Figure~\ref{fig:throughput-power-llama2}~(a), FluxBin demonstrates superior throughput across all batch sizes, with hybrid configurations (2b-s8/s16) matching the pure 2-bit baseline. Unlike conventional weight-only quantization whose advantage diminishes at large batch, FluxBin retains speedup because its LUT-based computation reduction ($\mu/K_b = 4\times$, Sec.~\ref{sec:accel-analysis}) is batch-size independent, and the batch dimension is tiled per thread block with fixed per-block shared memory cost.
To further isolate the decode phase, Figure~\ref{fig:throughput-power-llama2}~(c) reports decode-only throughput. FluxBin achieves $6.3\times$--$6.6\times$ decode speedup at BS=1, and the gap persists at large batch sizes, demonstrating that the LUT kernel delivers consistent acceleration throughout the autoregressive generation process.

\textbf{Memory Efficiency.} 
FluxBin achieves up to $4\times$ memory reduction as shown in Table~\ref{tab:max-memory-512} (Appendix~\ref{sec:appendix-memory}). This critical compression capability enables the feasible execution of massive models, such as LLaMA-2-70B, on a single A100 GPU, whereas the FP16 baseline fails due to Out-Of-Memory (OOM) errors.

\textbf{Energy Efficiency.} 
As visualized in Figure~\ref{fig:throughput-power-llama2}~(b), FluxBin significantly lowers the average power during inference (e.g., dropping from $\sim$225W to $\sim$115W on 13B). Combined with the speedup, this results in a total energy reduction of $10.19\times$ compared to FP16 (detailed in Appendix~\ref{sec:energy_utilization}). This efficiency stems from minimized data movement overhead and the reduction of floating-point arithmetic. %

\textbf{Architectural Robustness and Scalability.} 
To validate the generality of FluxBin, we extend our evaluation to the Qwen3 family (featuring GQA) and the massive LLaMA-3.1-70B (detailed in Appendix~\ref{sec:throughput_extension}). The results mirror our LLaMA-2 findings, and empirically confirms that FluxBin is architecture-agnostic, delivering consistent acceleration regardless of attention mechanisms or model size.
\begin{figure*}[t]
    \centering
    \includegraphics[width=0.8\linewidth]{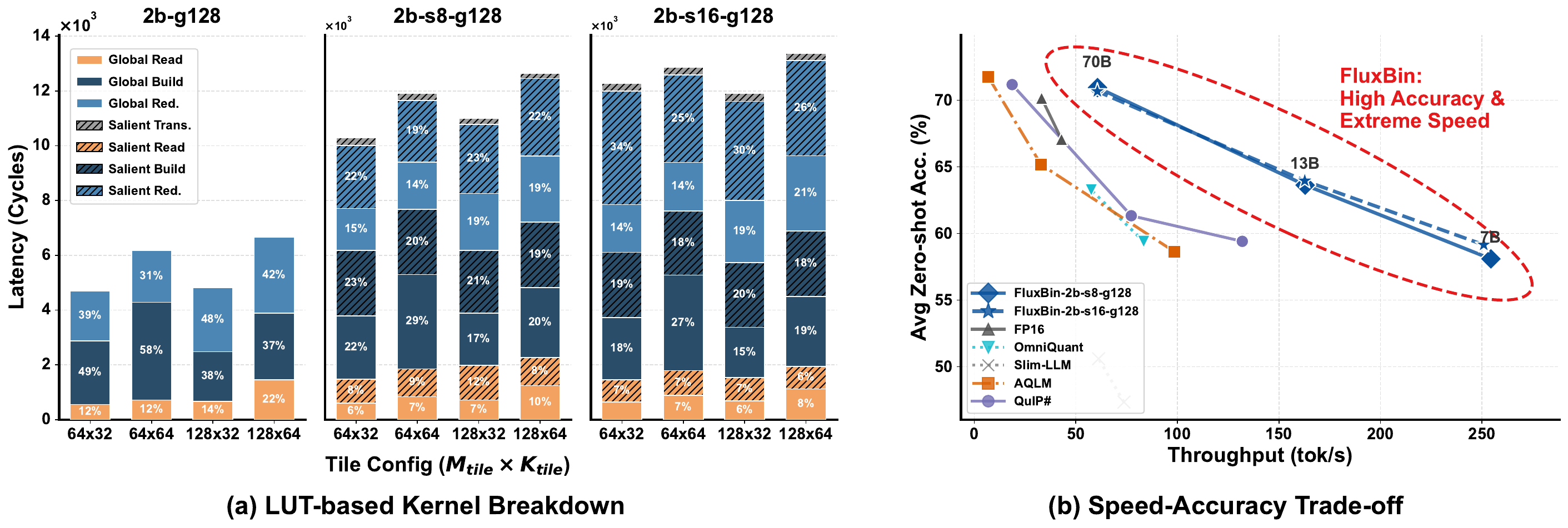}
    \vspace{-0.5em}
    \caption{(a) Latency breakdown of LUT-based kernel. (b) Speed-accuracy trade-off on LLaMA-2.}
    \label{fig:lut-profile-trade-off}
\end{figure*}

\subsection{Speed vs. Accuracy}
\label{sec:speed_acc}

Table~\ref{tab:llama2-task} presents a comparison of FluxBin with SOTA ultra-low-bit quantization methods~\cite{gptq, omni, slim-llm, aqlm, quip} across the LLaMA-2 family~\cite{llama2}. We analyze the results from accuracy retention and inference speed. During the speed evaluation, we set the input and output length to 512.

\textbf{Competitive Accuracy without Training.} Despite being a training-free method, FluxBin achieves accuracy comparable to that of fine-tuning or distillation based approaches. On LLaMA-2-7B, our {2b-s16-g128} configuration attains an average zero-shot accuracy of 59.15\%, which is on par with the heavily fine-tuned QuIP\# (59.41\%) and outperforms AQLM (58.64\%). %

\textbf{Extreme Speed Advantage.} For LLaMA-2-7B, FluxBin reaches {254.39 tokens/s}, representing a {5.92\(\times\) speedup} over FP16 baseline. This speed advantage scales effectively to larger models; on LLaMA-2-70B, FluxBin achieves {3.25\(\times\)} higher throughput compared to QuIP\#. %

\textbf{Optimal Trade-off.}
As visualized in Figure~\ref{fig:lut-profile-trade-off}~(b), existing methods typically compromise significant speed to maintain accuracy (bottom-right) or sacrifice accuracy for compression (top-left). FluxBin occupies the desirable {top-right region} of the scaling curve. This positions FluxBin as the optimal solution that simultaneously delivers high-fidelity generation and ultra-low latency, effectively breaking the conventional constraints of the speed-accuracy trade-off.

\textbf{Generalization Across Different Architectures.} 
To demonstrate FluxBin's broad applicability, we extend our evaluation to the latest LLaMA-3.1 and Qwen3 model families (detailed in Appendix~\ref{sec:extension_eval}).
In sensitive LLaMA-3.1 models, FluxBin exhibits remarkable robustness, achieving accuracy comparable to 3-bit baselines (e.g., GPTQ), outperforming AQLM (+17.5\% accuracy) while maintaining leadership in inference throughput. This suggests that as models become more complex and sensitive, the structural robustness of our Hybrid Base strategy becomes increasingly critical.

\subsection{Ablation Study}
\label{sec:ablation}

\textbf{Impact of Quantization Components.} The upper section of Table~\ref{tab:ablation} validates each algorithmic component. Removing {column scaling} significantly degrades PPL, while replacing {Saliency Search} with random selection causes a catastrophic failure (PPL 26.02), confirming the necessity of our Hessian-guided strategy. Relying solely on the Global Base increases PPL to 13.71, indicating that salient refinement is crucial for error compensation. %

\textbf{Impact of Kernel Optimization.}
The lower section of Table~\ref{tab:ablation} evaluates the benefits of the VCM strategy. Disabling the {VCM} strategy (treating refinement as sparse operations) results in an 18\% speed drop and a 33\% memory surge, validating VCM's ability to ensure coalesced memory access.

\begin{table}[t]
\centering
\scriptsize
\caption{Ablation study on LLaMA-2-7B (baseline: FluxBin-2b-s16-g128).
Removing a component is denoted by \XSolidBrush.}
\label{tab:ablation}
\setlength{\tabcolsep}{3pt}
\renewcommand{\arraystretch}{0.95}
\begin{tabular}{l c cc}
\toprule
\textbf{Config.} & \textbf{w/o} & \textbf{Wiki PPL} $\downarrow$ & \textbf{PIQA} $\uparrow$ \\
\midrule
Baseline &  & \textbf{8.20}  & \textbf{75.50} \\
Column Scaling   & \XSolidBrush & 8.96  & 71.65 \\
Saliency Search  & \XSolidBrush & 26.02 & 66.92 \\
Hybrid Base      & \XSolidBrush & 13.71 & 69.79 \\
\addlinespace[0.6ex]
\cmidrule(lr){1-4}
\addlinespace[0.6ex]
\textbf{Config.} & \textbf{w/o} & \textbf{Speed (tok/s)} $\uparrow$ & \textbf{Max Mem. (MB)} $\downarrow$ \\
\midrule
Baseline &  & \textbf{250.87} & \textbf{3475.21} \\
VCM      & \XSolidBrush & 205.68 & 4640.37 \\
\bottomrule
\end{tabular}
\end{table}

\textbf{Hyperparameter Sensitivity.} 

\textbf{(i)} {Group Size (g):} As shown in Sec.~\ref{sec:ablation_groupsize} Table~\ref{tab:ablation-groupsize}, reducing the group size can improve model performance but result in higher storage for the scaling factor. Empirically, g=128 yields the optimal trade-off, which balances the competitive accuracy and average bit. %

\textbf{(ii)} {Calibration Samples ($N_{cal}$):} Following the GPTQ framework, we estimate the Hessian using $N_{cal}$ random samples from the C4 dataset. As shown in Sec.~\ref {sec:ablation_nsample} Table~\ref{tab:nsample-ablation}, accuracy follows a rise-then-fall trend: $N_{cal}=256$ captures sufficient curvature information, whereas larger sizes ($>512$) lead to overfitting and performance drops.

\textbf{(iii)} {Sub-vector Size ($\mu$):} We demonstrate the impact of $\mu$ in Sec.~\ref{sec:shared_mem_analysis} and Table~\ref{tab:mu-profile}. The smaller $\mu$ minimizes the shared memory footprint, but increases the LUT reading operation, which is the main source of complexity. Despite larger LUTs, a larger $\mu$ without violating shared memory constraints achieves superior energy efficiency, and maintains almost constant inference latency with small $\mu$.   %

\subsection{Discussion: Sources of Acceleration}
\label{sec:accel-analysis}

We identify three sources behind FluxBin's speedup that collectively explain why the framework retains advantages even at large batch sizes.

\textbf{Reduced Memory Traffic.}
A 2-base binary configuration stores weights at $\sim$2 bits ($8\times$ compression vs.\ FP16), directly relieving the HBM bottleneck that dominates autoregressive decoding.

\textbf{Reduced and Simplified Operations.}
The LUT kernel replaces $\mu$ FMA operations per sub-vector with a single shared-memory lookup, reducing complexity by $\mu / K_b$ ($4\times$ at default $\mu{=}8$, $K_b{=}2$; Appendix~\ref{sec:appendix-complexity}). The remaining work is dominated by additions rather than FMAs, which enjoy higher pipeline throughput. Both reductions are batch-size independent.

\textbf{Trading ALU Computation for SRAM Access.}
Dequantization-based kernels spend ALU slots on bit-shift, mask, and type-convert instructions before every GEMV; FluxBin replaces this ALU-heavy path with shared-memory (SRAM) lookups. As SRAM bandwidth is independent of HBM, this offloads work from a contended resource (the ALU) without worsening the memory-bound bottleneck.

\textbf{Empirical Validation.}
Runtime profiling (Table~\ref{tab:GPU-profile}, Appendix~\ref{sec:energy_utilization}) corroborates the analysis above. FluxBin consistently reports $>$99.8\% GPU busy time versus 67--87\% for FP16. Its memory-bandwidth utilization is correspondingly low, reflecting the small weight traffic of the 2-bit format; together these indicate the LUT kernel is limited by on-chip lookup/compute throughput rather than stalled on HBM transfers (GPU busy time reflects occupancy, not arithmetic-unit saturation). The average power drop from 225W to $\sim$115W (13B) directly follows from replacing power-hungry FP16 FMAs with lightweight additions and SRAM lookups. Per-operator kernel latencies against QuIP\# and AQLM, together with a tiling-parameter study, are reported in Appendix~\ref{sec:kernel-latency} and Appendix~\ref{sec:kernel-tile}.
Figure~\ref{fig:lut-profile-trade-off}~(a) further breaks down per-block execution cycles: increasing $M_{\text{tile}}$ from 64 to 128 amortizes LUT Building cost (49\%$\to$38\%), and the VCM translation overhead for hybrid configurations (2b-s8/s16) remains negligible.

\section{Conclusion}

In this work, we propose {FluxBin} to bridge the gap between theoretical compression and realized acceleration. By synergizing {Row-Column Binary Decomposition} and {Hessian-guided Hybrid Bases} with a hardware-optimized {LUT-based Kernel} (LUT-BSF and VCM), we jointly enhance representational capacity and reduce runtime overhead. Across diverse architectures, FluxBin achieves up to \textbf{5.92$\times$ speedup} and \textbf{10.19$\times$ energy savings} while maintaining competitive accuracy, offering a training-free solution for efficient large-scale LLM deployment.

\section*{Limitations}
While FluxBin demonstrates significant advancements in ultra-low-bit LLM inference, we identify two primary limitations.

\textbf{Shared Memory Constraints on Scalability.}
The efficiency of our LUT-based kernel relies heavily on caching pre-computed results in on-chip Shared Memory (SRAM). However, the LUT storage requirement scales exponentially with the sub-vector size ($2^\mu$) and linearly with the batch size per block. While high-end GPUs like the A100 provide ample SRAM (164 KB/SM) to support our optimal configuration ($\mu=8$), deploying FluxBin on hardware with limited Shared Memory may necessitate reducing $\mu$ or the batch size.

\textbf{Optimization Bounds of Pure PTQ.}
FluxBin operates as a strict training-free framework to ensure rapid deployment. While it outperforms existing PTQ methods and rivals fine-tuned baselines, there remains a marginal accuracy gap compared to approaches that use extensive, computationally expensive iterative fine-tuning on specific tasks. Future work could explore integrating lightweight parameter-efficient fine-tuning on top of the binary bases to further bridge this gap without compromising the LUT-based inference speed.

\bibliography{references}

@misc{gpt3,
      title={Language Models are Few-Shot Learners}, 
      author={Tom B. Brown and Benjamin Mann and Nick Ryder and Melanie Subbiah and Jared Kaplan and Prafulla Dhariwal and Arvind Neelakantan and Pranav Shyam and Girish Sastry and Amanda Askell and Sandhini Agarwal and Ariel Herbert-Voss and Gretchen Krueger and Tom Henighan and Rewon Child and Aditya Ramesh and Daniel M. Ziegler and Jeffrey Wu and Clemens Winter and Christopher Hesse and Mark Chen and Eric Sigler and Mateusz Litwin and Scott Gray and Benjamin Chess and Jack Clark and Christopher Berner and Sam McCandlish and Alec Radford and Ilya Sutskever and Dario Amodei},
      year={2020},
      eprint={2005.14165},
      archivePrefix={arXiv},
      primaryClass={cs.CL},
      url={https://arxiv.org/abs/2005.14165}, 
}

@misc{llama2,
      title={Llama 2: Open Foundation and Fine-Tuned Chat Models}, 
      author={Hugo Touvron and Louis Martin and Kevin Stone and Peter Albert and Amjad Almahairi and Yasmine Babaei and Nikolay Bashlykov and Soumya Batra and Prajjwal Bhargava and Shruti Bhosale and Dan Bikel and Lukas Blecher and Cristian Canton Ferrer and Moya Chen and Guillem Cucurull and David Esiobu and Jude Fernandes and Jeremy Fu and Wenyin Fu and Brian Fuller and Cynthia Gao and Vedanuj Goswami and Naman Goyal and Anthony Hartshorn and Saghar Hosseini and Rui Hou and Hakan Inan and Marcin Kardas and Viktor Kerkez and Madian Khabsa and Isabel Kloumann and Artem Korenev and Punit Singh Koura and Marie-Anne Lachaux and Thibaut Lavril and Jenya Lee and Diana Liskovich and Yinghai Lu and Yuning Mao and Xavier Martinet and Todor Mihaylov and Pushkar Mishra and Igor Molybog and Yixin Nie and Andrew Poulton and Jeremy Reizenstein and Rashi Rungta and Kalyan Saladi and Alan Schelten and Ruan Silva and Eric Michael Smith and Ranjan Subramanian and Xiaoqing Ellen Tan and Binh Tang and Ross Taylor and Adina Williams and Jian Xiang Kuan and Puxin Xu and Zheng Yan and Iliyan Zarov and Yuchen Zhang and Angela Fan and Melanie Kambadur and Sharan Narang and Aurelien Rodriguez and Robert Stojnic and Sergey Edunov and Thomas Scialom},
      year={2023},
      eprint={2307.09288},
      archivePrefix={arXiv},
      primaryClass={cs.CL},
      url={https://arxiv.org/abs/2307.09288}, 
}

@INPROCEEDINGS{damf,
  author={Yang, Qingyao and Wang, Xiaoqin and Zhou, Yumei and Li, Qiang and Qiao, Shushan},
  booktitle={2025 IEEE International Symposium on Circuits and Systems (ISCAS)}, 
  title={Hardware Friendly Transformer Optimization with Dynamic Attention Matrix Fusion}, 
  year={2025},
  volume={},
  number={},
  pages={1-5},
  doi={10.1109/ISCAS56072.2025.11044099}}

@article{gptq,
  title={Gptq: Accurate post-training quantization for generative pre-trained transformers},
  author={Frantar, Elias and Ashkboos, Saleh and Hoefler, Torsten and Alistarh, Dan},
  journal={arXiv preprint arXiv:2210.17323},
  year={2022}
}

@article{awq,
author = {Lin, Ji and Tang, Jiaming and Tang, Haotian and Yang, Shang and Xiao, Guangxuan and Han, Song},
title = {AWQ: Activation-aware Weight Quantization for On-Device LLM Compression and Acceleration},
year = {2025},
issue_date = {December 2024},
publisher = {Association for Computing Machinery},
address = {New York, NY, USA},
volume = {28},
number = {4},
issn = {2375-0529},
month = jan,
pages = {12–17},
numpages = {6}
}

@misc{leiq,
      title={Exploring Layer-wise Information Effectiveness for Post-Training Quantization in Small Language Models}, 
      author={He Xiao and Qingyao Yang and Dirui Xie and Wendong Xu and Zunhai Su and Runming yang and Wenyong Zhou and Haobo Liu and Zhengwu Liu and Ngai Wong},
      year={2025},
      eprint={2508.03332},
      archivePrefix={arXiv},
      primaryClass={cs.LG},
      url={https://arxiv.org/abs/2508.03332}, 
}

@inproceedings{xnor-net_2016,
	location = {Cham},
	title = {{XNOR}-Net: {ImageNet} Classification Using Binary Convolutional Neural Networks},
	isbn = {978-3-319-46493-0},
	pages = {525--542},
	booktitle = {Computer Vision – {ECCV} 2016},
	publisher = {Springer International Publishing},
	author = {Rastegari, Mohammad and Ordonez, Vicente and Redmon, Joseph and Farhadi, Ali},
	editor = {Leibe, Bastian and Matas, Jiri and Sebe, Nicu and Welling, Max},
	date = {2016},
}

@inproceedings{QBB,
author = {Bulat, Adrian and Ouali, Yassine and Tzimiropoulos, Georgios},
title = {QBB: quantization with binary bases for LLMs},
year = {2024},
isbn = {9798331314385},
publisher = {Curran Associates Inc.},
address = {Red Hook, NY, USA},
booktitle = {Proceedings of the 38th International Conference on Neural Information Processing Systems},
articleno = {105},
numpages = {20},
location = {Vancouver, BC, Canada},
series = {NIPS '24}
}

@misc{ptqtp,
      title={PTQTP: Post-Training Quantization to Trit-Planes for Large Language Models}, 
      author={He Xiao and Runming Yang and Qingyao Yang and Wendong Xu and Zhen Li and Yupeng Su and Zhengwu Liu and Hongxia Yang and Ngai Wong},
      year={2026},
      eprint={2509.16989},
      archivePrefix={arXiv},
      primaryClass={cs.LG},
      url={https://arxiv.org/abs/2509.16989}, 
}

@inproceedings{DB-LLM,
    title = "{DB}-{LLM}: Accurate Dual-Binarization for Efficient {LLM}s",
    author = "Chen, Hong  and
      Lv, Chengtao  and
      Ding, Liang  and
      Qin, Haotong  and
      Zhou, Xiabin  and
      Ding, Yifu  and
      Liu, Xuebo  and
      Zhang, Min  and
      Guo, Jinyang  and
      Liu, Xianglong  and
      Tao, Dacheng",
    booktitle = "Findings of the Association for Computational Linguistics: ACL 2024",
    month = aug,
    year = "2024",
    address = "Bangkok, Thailand",
    publisher = "Association for Computational Linguistics",
    url = "https://aclanthology.org/2024.findings-acl.516/",
    doi = "10.18653/v1/2024.findings-acl.516",
    pages = "8719--8730",
}

@inproceedings{bi,
author = {Huang, Wei and Liu, Yangdong and Qin, Haotong and Li, Ying and Zhang, Shiming and Liu, Xianglong and Magno, Michele and Qi, Xiaojuan},
title = {BiLLM: pushing the limit of post-training quantization for LLMs},
year = {2024},
booktitle = {Proceedings of the 41st International Conference on Machine Learning},
location = {Vienna, Austria},
series = {ICML'24}
}

@article{ARB-LLM,
  title={ARB-LLM: Alternating Refined Binarizations for Large Language Models},
  author={Zhiteng Li and Xianglong Yan and Tianao Zhang and Haotong Qin and Dong Xie and Jiang Tian and Zhongchao Shi and Linghe Kong and Yulun Zhang and Xiaokang Yang},
  journal={ArXiv},
  year={2024},
  volume={abs/2410.03129},
  url={https://api.semanticscholar.org/CorpusID:273163233}
}

@article{HBLLM,
  title={HBLLM: Wavelet-Enhanced High-Fidelity 1-Bit Quantization for LLMs},
  author={Ningning Chen, Weicai Ye, Ying Jiang},
  journal={arXiv preprint arXiv:2512.00862},
  year={2025}
}

@article{pb-llm,
  title={PB-LLM: Partially Binarized Large Language Models},
  author={Yuzhang Shang and Zhihang Yuan and Qiang Wu and Zhen Dong},
  journal={ArXiv},
  year={2023},
  volume={abs/2310.00034},
  url={https://api.semanticscholar.org/CorpusID:263333921}
}

@inproceedings{SpQR,
  author       = {Tim Dettmers and
                  Ruslan Svirschevski and
                  Vage Egiazarian and
                  Denis Kuznedelev and
                  Elias Frantar and
                  Saleh Ashkboos and
                  Alexander Borzunov and
                  Torsten Hoefler and
                  Dan Alistarh},
  title        = {SpQR: {A} Sparse-Quantized Representation for Near-Lossless {LLM}
                  Weight Compression},
  booktitle    = {The Twelfth International Conference on Learning Representations,
                  {ICLR} 2024, Vienna, Austria, May 7-11, 2024},
  publisher    = {OpenReview.net},
  year         = {2024},
}

@inproceedings{Llm-mq,
  title={Llm-mq: Mixed-precision quantization for efficient llm deployment},
  author={Li, Shiyao and Ning, Xuefei and Hong, Ke and Liu, Tengxuan and Wang, Luning and Li, Xiuhong and Zhong, Kai and Dai, Guohao and Yang, Huazhong and Wang, Yu},
  booktitle={The Efficient Natural Language and Speech Processing Workshop with NeurIPS},
  volume={9},
  pages={3},
  year={2023}
}

@misc{slim-llm,
      title={SliM-LLM: Salience-Driven Mixed-Precision Quantization for Large Language Models}, 
      author={Wei Huang and Haotong Qin and Yangdong Liu and Yawei Li and Qinshuo Liu and Xianglong Liu and Luca Benini and Michele Magno and Shiming Zhang and Xiaojuan Qi},
      year={2025},
      eprint={2405.14917},
      archivePrefix={arXiv},
      url={https://arxiv.org/abs/2405.14917}, 
}

@inproceedings{LQ-Nets,
author = {Zhang, Dongqing and Yang, Jiaolong and Ye, Dongqiangzi and Hua, Gang},
title = {LQ-Nets: Learned Quantization for Highly Accurate and Compact Deep Neural Networks},
year = {2018},
isbn = {978-3-030-01236-6},
publisher = {Springer-Verlag},
address = {Berlin, Heidelberg},
url = {https://doi.org/10.1007/978-3-030-01237-3_23},
doi = {10.1007/978-3-030-01237-3_23},
booktitle = {Computer Vision – ECCV 2018: 15th European Conference, Munich, Germany, September 8-14, 2018, Proceedings, Part VIII},
pages = {373–390},
numpages = {18},
location = {Munich, Germany}
}

@INPROCEEDINGS{int8NVIDIA,
  author={Markidis, Stefano and Chien, Steven Wei Der and Laure, Erwin and Peng, Ivy Bo and Vetter, Jeffrey S.},
  booktitle={2018 IEEE International Parallel and Distributed Processing Symposium Workshops (IPDPSW)}, 
  title={NVIDIA Tensor Core Programmability, Performance \& Precision}, 
  year={2018},
  volume={},
  number={},
  pages={522-531},
  doi={10.1109/IPDPSW.2018.00091}}

@misc{zeng2023glm130,
      title={GLM-130B: An Open Bilingual Pre-trained Model}, 
      author={Aohan Zeng and Xiao Liu and Zhengxiao Du and Zihan Wang and Hanyu Lai and Ming Ding and Zhuoyi Yang and Yifan Xu and Wendi Zheng and Xiao Xia and Weng Lam Tam and Zixuan Ma and Yufei Xue and Jidong Zhai and Wenguang Chen and Peng Zhang and Yuxiao Dong and Jie Tang},
      year={2023},
      eprint={2210.02414},
      archivePrefix={arXiv},
      primaryClass={cs.CL},
      url={https://arxiv.org/abs/2210.02414}, 
}

@misc{lut-gemm,
      title={LUT-GEMM: Quantized Matrix Multiplication based on LUTs for Efficient Inference in Large-Scale Generative Language Models}, 
      author={Gunho Park and Baeseong Park and Minsub Kim and Sungjae Lee and Jeonghoon Kim and Beomseok Kwon and Se Jung Kwon and Byeongwook Kim and Youngjoo Lee and Dongsoo Lee},
      year={2024},
      eprint={2206.09557},
      archivePrefix={arXiv},
      primaryClass={cs.DC},
      url={https://arxiv.org/abs/2206.09557}, 
}

@inproceedings{lut-nn,
author = {Tang, Xiaohu and Wang, Yang and Cao, Ting and Zhang, Li Lyna and Chen, Qi and Cai, Deng and Liu, Yunxin and Yang, Mao},
title = {LUT-NN: Empower Efficient Neural Network Inference with Centroid Learning and Table Lookup},
year = {2023},
isbn = {9781450399906},
publisher = {Association for Computing Machinery},
address = {New York, NY, USA},
booktitle = {Proceedings of the 29th Annual International Conference on Mobile Computing and Networking},
articleno = {70},
numpages = {15},
location = {Madrid, Spain},
series = {ACM MobiCom '23}
}

@ARTICLE{UNPU,
  author={Lee, Jinmook and Kim, Changhyeon and Kang, Sanghoon and Shin, Dongjoo and Kim, Sangyeob and Yoo, Hoi-Jun},
  journal={IEEE Journal of Solid-State Circuits}, 
  title={UNPU: An Energy-Efficient Deep Neural Network Accelerator With Fully Variable Weight Bit Precision}, 
  year={2019},
  volume={54},
  number={1},
  pages={173-185},
  doi={10.1109/JSSC.2018.2865489}}

@misc{maleki2023lookupmaigemmincreasing,
      title={Look-Up mAI GeMM: Increasing AI GeMMs Performance by Nearly 2.5x via msGeMM}, 
      author={Saeed Maleki},
      year={2023},
      eprint={2310.06178},
      archivePrefix={arXiv},
      primaryClass={cs.PF},
      url={https://arxiv.org/abs/2310.06178}, 
}

@ARTICLE{9478922,
  author={Xu, Shiyu and Wang, Qi and Wang, Xingbo and Wang, Shihang and Ye, Terry Tao},
  journal={IEEE Transactions on Computer-Aided Design of Integrated Circuits and Systems}, 
  title={Multiplication Through a Single Look-Up-Table (LUT) in CNN Inference Computation}, 
  year={2022},
  volume={41},
  number={6},
  pages={1916-1928},
  doi={10.1109/TCAD.2021.3095825}}

@inproceedings{BiQGEMM,
author = {Jeon, Yongkweon and Park, Baeseong and Kwon, Se Jung and Kim, Byeongwook and Yun, Jeongin and Lee, Dongsoo},
title = {BiQGEMM: matrix multiplication with lookup table for binary-coding-based quantized DNNs},
year = {2020},
isbn = {9781728199986},
publisher = {IEEE Press},
booktitle = {Proceedings of the International Conference for High Performance Computing, Networking, Storage and Analysis},
articleno = {95},
numpages = {16},
location = {Atlanta, Georgia},
series = {SC '20}
}

@inproceedings{ft,
  title={QLoRA: Efficient Finetuning of Quantized LLMs},
  author={Dettmers, Tim and Pagnoni, Artidoro and Holtzman, Ari and Zettlemoyer, Luke},
  booktitle={Advances in Neural Information Processing Systems},
  year={2023}
}

@article{polino2018model,
  title={Model compression via distillation and quantization},
  author={Polino, Antonio and Pascanu, Razvan and Alistarh, Dan},
  journal={arXiv preprint arXiv:1802.05668},
  year={2018}
}

@INPROCEEDINGS{8578384,
  author={Jacob, Benoit and Kligys, Skirmantas and Chen, Bo and Zhu, Menglong and Tang, Matthew and Howard, Andrew and Adam, Hartwig and Kalenichenko, Dmitry},
  booktitle={2018 IEEE/CVF Conference on Computer Vision and Pattern Recognition}, 
  title={Quantization and Training of Neural Networks for Efficient Integer-Arithmetic-Only Inference}, 
  year={2018},
  volume={},
  number={},
  pages={2704-2713},
  doi={10.1109/CVPR.2018.00286}}

@inproceedings{LUT-tensor-core,
author = {Mo, Zhiwen and Wang, Lei and Wei, Jianyu and Zeng, Zhichen and Cao, Shijie and Ma, Lingxiao and Jing, Naifeng and Cao, Ting and Xue, Jilong and Yang, Fan and Yang, Mao},
title = {LUT Tensor Core: A Software-Hardware Co-Design for LUT-Based Low-Bit LLM Inference},
year = {2025},
isbn = {9798400712616},
publisher = {Association for Computing Machinery},
address = {New York, NY, USA},
url = {https://doi.org/10.1145/3695053.3731057},
doi = {10.1145/3695053.3731057},
booktitle = {Proceedings of the 52nd Annual International Symposium on Computer Architecture},
pages = {514–528},
numpages = {15},
location = {
},
series = {ISCA '25}
}

@inproceedings{BiBench,
author = {Qin, Haotong and Zhang, Mingyuan and Ding, Yifu and Li, Aoyu and Cai, Zhongang and Liu, Ziwei and Yu, Fisher and Liu, Xianglong},
title = {BiBench: benchmarking and analyzing network binarization},
year = {2023},
publisher = {JMLR.org},
booktitle = {Proceedings of the 40th International Conference on Machine Learning},
articleno = {1177},
numpages = {38},
location = {Honolulu, Hawaii, USA},
series = {ICML'23}
}

@inproceedings{OBQ,
author = {Frantar, Elias and Singh, Sidak Pal and Alistarh, Dan},
title = {Optimal brain compression: a framework for accurate post-training quantization and pruning},
year = {2022},
isbn = {9781713871088},
publisher = {Curran Associates Inc.},
address = {Red Hook, NY, USA},
booktitle = {Proceedings of the 36th International Conference on Neural Information Processing Systems},
articleno = {323},
numpages = {14},
location = {New Orleans, LA, USA},
series = {NIPS '22}
}

@INPROCEEDINGS{6710599,
  author={Krishnamoorthy, Aravindh and Menon, Deepak},
  booktitle={2013 Signal Processing: Algorithms, Architectures, Arrangements, and Applications (SPA)}, 
  title={Matrix inversion using Cholesky decomposition}, 
  year={2013},
  volume={},
  number={},
  pages={70-72},
  doi={}}

@inproceedings{SparseGPT,
author = {Frantar, Elias and Alistarh, Dan},
title = {SparseGPT: massive language models can be accurately pruned in one-shot},
year = {2023},
publisher = {JMLR.org},
booktitle = {Proceedings of the 40th International Conference on Machine Learning},
articleno = {414},
numpages = {15},
location = {Honolulu, Hawaii, USA},
series = {ICML'23}
}

@inproceedings{wolf2020transformers,
  title={Transformers: State-of-the-art natural language processing},
  author={Wolf, Thomas and Debut, Lysandre and Sanh, Victor and Chaumond, Julien and Delangue, Clement and Moi, Anthony and Cistac, Pierric and Rault, Tim and Louf, Remi and Funtowicz, Morgan and others},
  booktitle={Proceedings of the 2020 conference on empirical methods in natural language processing: system demonstrations},
  pages={38--45},
  year={2020}
}

@inproceedings{omni,
 author = {Shao, Wenqi and Chen, Mengzhao and Zhang, Zhaoyang and Xu, Peng and Zhao, Lirui and Li, Zhiqian and Zhang, Kaipeng and Peng, Gao and Qiao, Yu and Luo, Ping},
 booktitle = {International Conference on Representation Learning},
 editor = {B. Kim and Y. Yue and S. Chaudhuri and K. Fragkiadaki and M. Khan and Y. Sun},
 pages = {45472--45496},
 title = {OmniQuant: Omnidirectionally Calibrated Quantization for Large Language Models},
 volume = {2024},
 year = {2024}
}

@inproceedings{aqlm,
author = {Egiazarian, Vage and Panferov, Andrei and Kuznedelev, Denis and Frantar, Elias and Babenko, Artem and Alistarh, Dan},
title = {Extreme compression of large language models via additive quantization},
year = {2024},
publisher = {JMLR.org},
booktitle = {Proceedings of the 41st International Conference on Machine Learning},
articleno = {488},
numpages = {20},
location = {Vienna, Austria},
series = {ICML'24}
}

@inproceedings{quip,
author = {Tseng, Albert and Chee, Jerry and Sun, Qingyao and Kuleshov, Volodymyr and De Sa, Christopher},
title = {QuIP\#: even better LLM quantization with hadamard incoherence and lattice codebooks},
year = {2024},
publisher = {JMLR.org},
booktitle = {Proceedings of the 41st International Conference on Machine Learning},
articleno = {1987},
numpages = {27},
location = {Vienna, Austria},
series = {ICML'24}
}

@misc{FlexRound,
  title={FlexRound: Learnable Rounding based on Element-wise Division for Post-Training Quantization}, 
  author={Jung Hyun Lee and Jeonghoon Kim and Se Jung Kwon and Dongsoo Lee},
  year={2023},
  eprint={2306.00317},
  archivePrefix={arXiv},
  primaryClass={cs.LG},
  url={https://arxiv.org/abs/2306.00317}, 
}

\appendix

\section{Derivation of Parameter Optimization}
\label{sec:derivation}

In this section, we provide the rigorous mathematical proofs for the optimal update rules presented in the Methodology section.

\subsection{Derivation of Optimal Scaling Factors}

Recall the objective function for a single decomposition order \(k\):
\begin{equation}
\mathcal{L} = \left\|\mathbf{W} - \left(\boldsymbol{\alpha}_r \boldsymbol{\alpha}_c^\top\right) \odot \mathbf{B}\right\|_F^2
\label{eq:appendix_loss}
\end{equation}

\begin{proof}[\textbf{Proof of Proposition~3.2.}]
We aim to derive the closed-form solution for the row scaling vector \(\boldsymbol{\alpha}_r\) while keeping \(\mathbf{B}\) and \(\boldsymbol{\alpha}_c\) fixed.
Since the Frobenius norm is the sum of squared element-wise errors, Eq.~\eqref{eq:appendix_loss} can be expanded as:
\begin{equation}
\mathcal{L} = \sum_{i=1}^{M} \sum_{j=1}^{N} \left( \mathbf{W}_{i j} - \boldsymbol{\alpha}_{r, i} \cdot \boldsymbol{\alpha}_{c, j} \cdot \mathbf{B}_{i j} \right)^2
\end{equation}
The optimization problem for \(\boldsymbol{\alpha}_r\) decouples into \(M\) independent scalar least-squares problems. For the \(i\)-th row scalar \(\boldsymbol{\alpha}_{r, i}\), the partial derivative is:
\begin{equation}
\begin{aligned}
\frac{\partial \mathcal{L}}{\partial \boldsymbol{\alpha}_{r, i}} &= \frac{\partial}{\partial \boldsymbol{\alpha}_{r, i}} \sum_{j=1}^{N} \left( \mathbf{W}_{i j} - \boldsymbol{\alpha}_{r, i} (\boldsymbol{\alpha}_{c, j} \mathbf{B}_{i j}) \right)^2 \\
&= -2 \sum_{j=1}^{N} \left( \mathbf{W}_{i j} - \boldsymbol{\alpha}_{r, i} (\boldsymbol{\alpha}_{c, j} \mathbf{B}_{i j}) \right) \cdot (\boldsymbol{\alpha}_{c, j} \mathbf{B}_{i j})
\end{aligned}
\end{equation}
Setting the derivative to zero to find the critical point:
\begin{equation}
\sum_{j=1}^{N} \mathbf{W}_{i j} (\boldsymbol{\alpha}_{c, j} \mathbf{B}_{i j}) = \boldsymbol{\alpha}_{r, i} \sum_{j=1}^{N} (\boldsymbol{\alpha}_{c, j} \mathbf{B}_{i j})^2
\end{equation}
Solving for \(\boldsymbol{\alpha}_{r, i}\), we obtain:
\begin{equation}
\boldsymbol{\alpha}_{r, i} = \frac{\sum_{j=1}^{N} (\boldsymbol{\alpha}_{c, j} \cdot \mathbf{B}_{i j}) \cdot \mathbf{W}_{i j}}{\sum_{j=1}^{N} (\boldsymbol{\alpha}_{c, j} \cdot \mathbf{B}_{i j})^2 + \epsilon}
\end{equation}
where \(\epsilon\) is added for numerical stability. This matches Eq.~\eqref{equa:row-scale-update}.
Due to the symmetry of the rank-1 structure, the derivation for the column scale \(\boldsymbol{\alpha}_c\) follows an identical logic by taking the derivative w.r.t \(\boldsymbol{\alpha}_{c, j}\), yielding the column-scale expression in Eq.~\eqref{equa:row-scale-update}.
\end{proof}

\subsection{Derivation of Joint Binary Basis Optimization}

When \(K_b \ge 2\), we optimize the binary bases \(\mathbf{B}^{(1:K_b)}\) jointly.

\begin{proof}[\textbf{Proof of Proposition~3.3.}]
We define the magnitude tensor for the \(k\)-th basis as \(\mathbf{M}^{(k)} = \boldsymbol{\alpha}_r^{(k)} \boldsymbol{\alpha}_c^{(k)\top}\). The global reconstruction error is:
\begin{equation}
\mathcal{L} = \left\| \mathbf{W} - \sum_{k=1}^{K_b} \mathbf{M}^{(k)} \odot \mathbf{B}^{(k)} \right\|_F^2
\end{equation}
Crucially, the elements of \(\mathbf{B}^{(k)}\) are independent across spatial coordinates \((i,j)\). Thus, minimizing the global Frobenius norm is equivalent to minimizing the squared error for each element independently:
\begin{equation}
\fiteq{
\min_{\{\mathbf{B}^{(k)}\}} \mathcal{L} \iff \sum_{i,j} \min_{b_{ij}^{(1)}, \dots, b_{ij}^{(K_b)}} \left( \mathbf{W}_{ij} - \sum_{k=1}^{K_b} \mathbf{M}_{ij}^{(k)} \cdot b_{ij}^{(k)} \right)^2
}
\end{equation}
For a specific coordinate \((i,j)\), let \(\mathbf{b} = [b^{(1)}, \dots, b^{(K_b)}]^\top \in \{-1, +1\}^{K_b}\) be the vector of binary choices. The problem reduces to:
\begin{equation}
\fiteq{
\mathbf{b}^* = \underset{\mathbf{b} \in\{-1,+1\}^{K_b}}{\operatorname{argmin}} \left( \mathbf{W}_{i j} - \sum_{k=1}^{K_b} \mathbf{M}_{i j}^{(k)} \cdot b^{(k)} \right)^2
}
\end{equation}
Since the square function is monotonic for positive errors, minimizing the squared L2 error is equivalent to minimizing the absolute L1 error in this scalar context. Thus, we have:
\begin{equation}
\fiteq{
\mathbf{B}_{i j}^{(1: K_b)} = \underset{\mathbf{b} \in\{-1,+1\}^{K_b}}{\operatorname{argmin}} \left| \mathbf{W}_{i j} - \sum_{k=1}^{K_b} \mathbf{M}_{i j}^{(k)} \cdot b_k \right|
}
\end{equation}
This confirms that the element-wise exhaustive search (Eq.~\eqref{equa:B-update}) yields the globally optimal binary assignment for the fixed scaling factors.
\end{proof}
\section{Algorithmic Descriptions}
\label{sec:algo_desc}

FluxBin is seamlessly integrated into the GPTQ framework~\cite{gptq}, adopting a block-wise reconstruction paradigm to ensure high inference efficiency and quantization fidelity. By leveraging the Optimal Brain Quantization (OBQ) theory, we utilize the inverse Hessian matrix $\mathbf{H}^{-1}$ not only for structural saliency search but also for compensating the quantization error accumulating across columns. The Hessian is computed efficiently using calibration data via $\mathbf{H} = 2\mathbf{X}\mathbf{X}^T$, and its inverse is updated iteratively, ensuring the process remains computationally lightweight.

We provide two algorithmic variants corresponding to the configurations evaluated in the main paper:

\textbf{FluxBin Base (Algorithm~\ref{alg:fluxbin_base}).} 
This algorithm describes the standard global base quantization process without mixed-precision refinement. It iterates through column groups, performing the proposed {Row-Column Binary Decomposition} to obtain the quantized block $\widehat{\mathbf{W}}_{blk}$. The residual error is then propagated to the remaining unquantized weights using the Hessian inverse. This algorithm is utilized for pure binary configurations, specifically {2b-g128} ($K_b=2$) and {4b-g128} ($K_b=4$), where no salient refinement is applied.

\textbf{FluxBin Hybrid Base (Algorithm~\ref{alg:fluxbin_hybrid}).} 
This algorithm details the salient-aware hybrid base (mixed-precision) strategy. Within each block, it first identifies critical columns via {StructuralSearch} using Hessian metrics. Following the global base decomposition, it performs a secondary decomposition on the residual $\mathbf{R}_{blk}$ restricted to the masked salient columns. This algorithm is utilized for our primary hybrid configurations, specifically {2b-s8-g128} ($s=8$) and {2b-s16-g128} ($s=16$).

\textbf{Sub-routines (Algorithm~\ref{alg:subroutines}).} 
This section defines the core components: {RC\_BinaryDecomp}, which solves the optimization problem for decoupled scaling factors and binary bases (Eq.~\eqref{equa:row-scale-update}-\eqref{equa:B-update}), and {StructuralSearch}, which computes the column-wise sensitivity score (Eq.~\eqref{equa:saliency-score-column}).

\begin{figure*}[h]
    \centering
    \small
    
    \begin{minipage}[t]{0.48\linewidth}
        \begin{algorithm}[H]
            \caption{FluxBin-Base (Global Base)}
            \label{alg:fluxbin_base}
            \begin{algorithmic}[1]
                \STATE \textbf{Input:} $\mathbf{W}$, Calib $\mathbf{X}$, Group $g$, Bases $K_b$, Iter $T$
                \STATE \textbf{Output:} $\widehat{\mathbf{W}}$, Params $\Theta$
                
                \STATE $\mathbf{H} \leftarrow 2\mathbf{X}\mathbf{X}^T$ \COMMENT{Hessian from input}
                \STATE $\mathbf{H}^{-1} \leftarrow \text{CholeskyInverse}(\mathbf{H})$
                \STATE $\mathbf{Q} \leftarrow \mathbf{0}$
                
                \FOR{$j = 0, g, \dots, N-g$}
                    \STATE $\mathbf{W}_{blk} \leftarrow \mathbf{W}_{:, j:j+g}$
                    \STATE $\mathbf{H}^{-1}_{blk} \leftarrow (\mathbf{H}^{-1})_{j:j+g, j:j+g}$
                    
                    \STATE \COMMENT{Global Base Decomp}
                    \STATE $\widehat{\mathbf{W}}_{blk}, \Theta \leftarrow \text{\bf RC\_Decomp}(\mathbf{W}_{blk}, \mathbf{1}, K_b, T)$
                    
                    \STATE $\mathbf{Q}_{:, j:j+g} \leftarrow \widehat{\mathbf{W}}_{blk}$
                    
                    \STATE \COMMENT{Error Compensation}
                    \STATE $\mathbf{E} \leftarrow \mathbf{W}_{blk} - \widehat{\mathbf{W}}_{blk}$
                    \STATE Update $\mathbf{W}_{remaining}$ via $\mathbf{E}, \mathbf{H}^{-1}$
                \ENDFOR
                \STATE \textbf{return} $\mathbf{Q}, \Theta$
            \end{algorithmic}
        \end{algorithm}
    \end{minipage}
    \hfill
    \begin{minipage}[t]{0.48\linewidth}
        \begin{algorithm}[H] 
            \caption{FluxBin-Hybrid (Salient-Aware)}
            \label{alg:fluxbin_hybrid}
            \begin{algorithmic}[1]
                \STATE \textbf{Input:} $\mathbf{W}, \mathbf{X}, g$, Bases $K_b, K_s$, Salient $N_{sal}, T$
                \STATE \textbf{Output:} $\widehat{\mathbf{W}}$, $\Theta$, $\Phi$, Mask $\mathcal{M}$
                
                \STATE $\mathbf{H} \leftarrow 2\mathbf{X}\mathbf{X}^T$; $\mathbf{H}^{-1} \leftarrow \text{CholInv}(\mathbf{H})$
                \FOR{$j = 0, g, \dots, N-g$}
                    \STATE $\mathbf{W}_{blk} \leftarrow \mathbf{W}_{:, j:j+g}$
                    \STATE $\mathbf{H}^{-1}_{blk} \leftarrow (\mathbf{H}^{-1})_{j:j+g, j:j+g}$
                    
                    \STATE \COMMENT{Step 1: Saliency Search}
                    \STATE $\mathcal{I} \leftarrow \text{\bf StructuralSearch}(\mathbf{W}_{blk}, \mathbf{H}^{-1}_{blk}, N_{sal})$
                    \STATE Construct mask $\mathcal{M}_{blk}$ from $\mathcal{I}$
                    
                    \STATE \COMMENT{Step 2: Global Base}
                    \STATE $\widehat{\mathbf{W}}_{g}, \Theta \leftarrow \text{\bf RC\_Decomp}(\mathbf{W}_{blk}, \mathbf{1}, K_b, T)$
                    
                    \STATE \COMMENT{Step 3: Salient Refinement}
                    \STATE $\mathbf{R}_{blk} \leftarrow \mathbf{W}_{blk} - \widehat{\mathbf{W}}_{g}$
                    \STATE $\widehat{\mathbf{W}}_{s}, \Phi \leftarrow \text{\bf RC\_Decomp}(\mathbf{R}_{blk}, \mathcal{M}_{blk}, K_s, T)$
                    
                    \STATE $\widehat{\mathbf{W}}_{blk} \leftarrow \widehat{\mathbf{W}}_{g} + \widehat{\mathbf{W}}_{s}$
                    
                    \STATE \COMMENT{Step 4: Compensation}
                    \STATE $\mathbf{E} \leftarrow \mathbf{W}_{blk} - \widehat{\mathbf{W}}_{blk}$
                    \STATE Update $\mathbf{W}_{remaining}$ via $\mathbf{E}, \mathbf{H}^{-1}$
                \ENDFOR
                \STATE \textbf{return} $\widehat{\mathbf{W}}, \Theta, \Phi, \mathcal{M}$
            \end{algorithmic}
        \end{algorithm}
    \end{minipage}
\end{figure*}

\begin{algorithm}[h]
  \caption{FluxBin Sub-routines}
  \label{alg:subroutines}
  
  \begin{minipage}[t]{0.54\linewidth}
    \begin{algorithmic}[1]
        \STATE \textbf{Function} RC\_Decomp($\mathbf{X}, \text{mask}, K, T$)
        \STATE \quad Init residual $\mathbf{R} \leftarrow \mathbf{X} \odot \text{mask}$
        \STATE \quad $\widehat{\mathbf{X}} \leftarrow \mathbf{0}$
        
        \STATE \quad // \textit{Init via Greedy Residual}
        \FOR{$k = 1$ \textbf{to} $K$}
            \STATE Init $\mathbf{B}^{(k)}, \boldsymbol{\alpha}_r^{(k)}, \boldsymbol{\alpha}_c^{(k)}$ from $\mathbf{R}$
            \STATE $\mathbf{R} \leftarrow \mathbf{R} - (\boldsymbol{\alpha}_r^{(k)} \boldsymbol{\alpha}_c^{(k)\top}) \odot \mathbf{B}^{(k)}$
        \ENDFOR
          
        \STATE \quad // \textit{Alternating Optimization}
        \FOR{$iter = 1$ \textbf{to} $T$}
            \FOR{$k = 1$ \textbf{to} $K$}
                \STATE Update $\boldsymbol{\alpha}_r^{(k)}$ (Eq.~\ref{equa:row-scale-update})
                \STATE Update $\boldsymbol{\alpha}_c^{(k)}$ (Eq.~\ref{equa:row-scale-update})
            \ENDFOR
            \STATE Update $\{\mathbf{B}^{(k)}\}$ via search (Eq.~\ref{equa:B-update})
            \FOR{$k = 1$ \textbf{to} $K$}
                \STATE $\mathbf{B}^{(k)} \leftarrow \mathbf{B}^{(k)} \odot \text{mask}$ 
            \ENDFOR
        \ENDFOR
        
        \STATE \quad Reconstruct $\widehat{\mathbf{X}}$
        \STATE \quad \textbf{return} $\widehat{\mathbf{X}}, \{\boldsymbol{\alpha}, \mathbf{B}\}$
        \STATE \textbf{End Function}
    \end{algorithmic}
  \end{minipage}%
  \hfill
  \begin{minipage}[t]{0.44\linewidth}
    \begin{algorithmic}[1]
        \STATE \textbf{Function} StructuralSearch
        \STATE \quad \textbf{Input:} $\mathbf{W}, \mathbf{H}^{-1}, N_{sal}$
        \STATE \quad // \textit{Compute Sensitivity}
        \FOR{$j = 1$ \textbf{to} $g$}
            \STATE $S_j = \sum_{i} \frac{\mathbf{W}_{ij}^2}{[\mathbf{H}^{-1}]_{jj}^2}$
        \ENDFOR
        \STATE \quad $\mathcal{I} \leftarrow \text{TopK}(\{S_j\}, N_{sal})$
        \STATE \quad \textbf{return} $\mathcal{I}$
        \STATE \textbf{End Function}
    \end{algorithmic}
  \end{minipage}
\end{algorithm}

\section{Cost Estimation}

\subsection{Compute Complexity Analysis of LUT-base Kernel}
\label{sec:appendix-complexity}

To quantify the computational efficiency of FluxBin, we analyze the theoretical complexity of our LUT-based kernel. Following the standard analysis in LUT-GEMM~\cite{lut-gemm}, the total cost \(C\) consists of the LUT construction cost (\(C_{build}\)) and the table lookup cost (\(C_{read}\)).

Assume the weight matrix has dimensions \(M \times N\), the sub-vector length is \(\mu\), and the number of binary bases is \(K\). Since the column scaling factor \(\boldsymbol{\alpha}_c^{(k)}\) is unique for each basis \(k\) and is fused into the input, independent LUTs must be constructed for each basis.

\textbf{1. Global Base Complexity.} 
For the global base branch with \(K_b\) bases, the input dimension is \(N\).
\begin{equation}
\begin{aligned}
C_{build}^{\text{base}} &= K_b \cdot \frac{N}{\mu} \cdot 2^\mu \\
C_{read}^{\text{base}} &= K_b \cdot M \cdot \frac{N}{\mu}
\end{aligned}
\end{equation}
Here, \(\frac{N}{\mu}\) represents the number of sub-vectors, and \(2^\mu\) is the number of entries per LUT.

\textbf{2. Salient Refinement Complexity.} 
For the salient refinement branch, the operation is restricted to \(N_{sal}\) columns with \(K_s\) bases.
\begin{equation}
\begin{aligned}
C_{build}^{\text{sal}} &= K_s \cdot \frac{N_{sal}}{\mu} \cdot 2^\mu \\
C_{read}^{\text{sal}} &= K_s \cdot M \cdot \frac{N_{sal}}{\mu}
\end{aligned}
\end{equation}

\textbf{3. Total Complexity.} 
The overall complexity is the sum of both branches:
\begin{equation}
\fiteq{
C_{total} = \underbrace{\frac{2^\mu}{\mu} (K_b N + K_s N_{sal})}_{\text{LUT Build}} + \underbrace{\frac{M}{\mu} (K_b N + K_s N_{sal})}_{\text{LUT Read}}
}
\label{equa:complexity-full}
\end{equation}

In typical generative LLM scenarios, the output channel dimension \(M\) is significantly larger than the LUT size \(2^\mu\) (e.g., \(M \gg 2^\mu\)). Consequently, the lookup cost \(C_{read}\) dominates the total runtime. The complexity can be approximated as:
\begin{equation}
C_{total} \approx \mathcal{O}\left( \frac{M}{\mu} (K_b N + K_s N_{sal}) \right)
\label{equa:complexity-approx}
\end{equation}

\textbf{Efficiency Gain.} 
Compared to standard FP16 matrix multiplication which has a complexity of \(\mathcal{O}(MN)\), FluxBin achieves a theoretical reduction factor of approximately \(\frac{\mu}{K_{b}}\). Furthermore, since the salient refinement typically involves a very sparse subset of columns (\(N_{sal} \ll N\)), the additional overhead introduced by the hybrid precision is marginal, preserving the high-throughput advantage of the LUT-based paradigm.

\subsection{Storage Consumption Analysis}
\label{sec:appendix-storage}

We conduct a comprehensive analysis of the storage overhead introduced by FluxBin, evaluating the Total Storage (\(S_{\text{total}}\)) and the Effective Average Bit-Width (ABW). Let \(M\) and \(N\) denote the number of rows and columns of the weight matrix \(\mathbf{W}\), respectively. Our configuration divides columns into groups of size \(g\), and selects \(s\) salient columns per group.

The quantization model consists of \(K_b\) binary bases for the global approximation and \(K_s\) bases for the salient refinement. We assume all scaling factors (row and column) and indices are stored in FP16 format (16 bits).

\textbf{1. Global Base Storage (\(S_{\text{base}}\)).}
The global base component stores the dense binary weights, column-wise scaling factors, and block-wise row scaling factors. The storage cost (in bits) is calculated as:
\begin{equation}
\fiteq{
S_{\text{base}} = K_b \cdot \left( \underbrace{MN}_{\text{Weights}} + \underbrace{16N}_{\text{Col Scales}} + \underbrace{16 \cdot M \cdot \frac{N}{g}}_{\text{Row Scales}} \right)
}
\label{eq:storage_base}
\end{equation}
where the term \(M \cdot \frac{N}{g}\) reflects our strategy of storing one row scale per group to capture local magnitude variations.

\textbf{2. Salient Refinement Storage (\(S_{\text{sal}}\)).}
The salient refinement targets a subset of columns, denoted as \(N_{\text{sal}} = \frac{N}{g} \cdot s\). This component requires storage for the sparse binary refinements, their corresponding scales, and the \textit{indexing metadata}.
\begin{equation}
\fiteq{
S_{\text{sal}} = K_s \cdot \left( MN_{\text{sal}} + 16N_{\text{sal}} + 16 \cdot M \cdot \frac{N}{g} \right) + S_{\text{idx}}
}
\label{eq:storage_sal}
\end{equation}
The term \(S_{\text{idx}}\) accounts for the overhead of storing the indices of salient columns. Since we select \(s\) indices for each of the \(N/g\) groups, the overhead is:
\begin{equation}
S_{\text{idx}} = 16 \cdot N_{\text{sal}}
\end{equation}
Note that we use 16 bits per index to ensure memory alignment and support large column dimensions, although 8 bits could suffice for smaller group sizes (e.g., \(g \le 256\)).

\textbf{3. Effective Average Bit-Width (ABW).}
Aggregating the components from Eq.~\eqref{eq:storage_base} and Eq.~\eqref{eq:storage_sal}, the total storage is \(S_{\text{total}} = S_{\text{base}} + S_{\text{sal}}\). The effective average bit-width is derived as:
\begin{equation}
\text{ABW} = \frac{S_{\text{total}}}{MN} = \frac{S_{\text{base}} + S_{\text{sal}}}{MN}
\label{eq:abw}
\end{equation}

For our typical configuration (e.g., \(K_b=2\), \(K_s=2\), \(g=128\), \(s=8\)), the salient branch introduces a sparsity ratio of \(s/g = 6.25\%\). The indexing overhead \(S_{\text{idx}}\) is amortized over the large matrix dimensions (\(\approx \frac{16 \cdot s/g}{M}\) bits per weight), so the effective ABW settles at \(\approx 2.63\) bits (and \(\approx 2.75\) bits for \(s=16\)) while significantly boosting model accuracy.

\subsection{Maximum Memory Consumption}
\label{sec:appendix-memory}
Table \ref{tab:max-memory-512} presents the peak GPU memory allocation during inference.

\textbf{Comparison with FP16 Baseline:} FluxBin significantly reduces memory footprints compared to the FP16 baseline. For instance, on Qwen3-32B, the memory usage drops from 62.7GB (FP16) to 12.1GB (2b-g128), achieving a \(5.1 \times \)compression ratio. Notably, for LLaMA-2-70B, which triggers an Out-Of-Memory (OOM) error in FP16 on our testing hardware, FluxBin enables successful inference with only 20.2GB of memory in the 2-bit configuration, demonstrating its capability to democratize large model deployment on consumer-grade GPUs. Furthermore, FluxBin consistently achieves comparable memory usage to other 2-bit baselines like AQLM and QuIP\#.

\textbf{Impact of Model Configuration:}
The memory consumption correlates with the complexity of the binary decomposition:
\begin{itemize}
\item Binary Bases: Increasing the global quantization precision from 2-bit to 4-bit naturally increases memory usage. For LLaMA-2-7B, the usage rises from 2.9GB (2b-g128) to 4.7GB (4b-g128), reflecting the storage of additional binary bases.
\item Salient Refinement: Incorporating salient refinement (s8/s16) introduces a moderate memory overhead compared to the standard 2b configuration due to the storage of the auxiliary dense matrix and indices. However, it is worth noting that the difference between 2b-s8 and 2b-s16 is negligible (e.g., \(< 1\) MB for LLaMA-2-7B). This confirms that our Virtual Columnar Mapping strategy is highly memory-efficient: doubling the number of salient columns incurs minimal additional cost, as the overhead is dominated by fixed allocations (e.g., KV cache and activation buffers) rather than the compact salient weights themselves.
\end{itemize}
\begin{table*}[!htbp]
  \centering
  \caption{Max GPU Memory (MB) Allocated in Inference with Input Length 512 and Output Length 512.}
  \setlength{\tabcolsep}{6pt}
  \resizebox{\textwidth}{!}{\begin{tabular}{lccc cccc}
    \toprule
    & \multicolumn{3}{c}{\textbf{Baselines}}
    & \multicolumn{4}{c}{\textbf{FluxBin}} \\
    \cmidrule(lr){2-4} \cmidrule(lr){5-8}

    \textbf{Model Name}
    & 
    \textbf{FP16}
    & \textbf{AQLM 2bit}
    & \textbf{QuIP\# 2bit}
    & \textbf{2b-g128}
    & \textbf{2b-s8-g128}
    & \textbf{2b-s16-g128}
    & \textbf{4b-g128} \\
    \midrule

    LLaMA-2-7B
    & 13330.61 & 5247.35 & 3461.87
    & 2895.24 & 3474.30 & 3475.21 & 4684.38 \\

    LLaMA-2-13B
    & 25681.85 & 4058.02 & 5785.77
    & 4558.59 & 6172.64 & 6173.32 & 8474.27 \\

    LLaMA-2-70B
    & OOM & 20415.70 & 19360.84
    & 20251.61 & 26393.52 & 26393.13 & 38633.95 \\

    Qwen3-8B
    & 15779.09 & -- & --
    & 4557.90 & 5185.48 & 5187.33 & 6444.57 \\

    Qwen3-14B
    & 28322.55 & -- & --
    & 7043.43 & 8215.49 & 8221.31 & 10486.19 \\

    Qwen3-32B
    & 62720.77 & -- & --
    & 12135.80 & 14993.71 & 14995.25 & 20323.56 \\

    \bottomrule
  \end{tabular}}
  \label{tab:max-memory-512}
\end{table*}

\section{Extension Experiment}

\subsection{Extension on Throughput}
\label{sec:throughput_extension}

To comprehensively validate the generality of FluxBin, we extend our throughput evaluation to two representative model families: Qwen3 (8B, 14B, 32B), which features Grouped-Query Attention (GQA), and LLaMA-3.1 (8B, 70B), representing the latest generation of massive-scale models.

As illustrated in Figure~\ref{fig:throughput-llama3} and Figure~\ref{fig:throughput-qwen3}, the results exhibit remarkable consistency. Crucially, the throughput for hybrid precision configurations ({2b-s8-g128} and {2b-s16-g128}) remains virtually indistinguishable from the pure 2-bit baseline ({2b-g128}) across all cases. This empirical evidence reinforces two critical attributes of FluxBin:

\begin{itemize}
    \item \textbf{Architectural Robustness:} FluxBin proves to be architecture-agnostic. Whether applied to standard architectures or modern designs utilizing GQA (as in Qwen3), our kernel integrates seamlessly, delivering consistent acceleration without being hindered by specific architectural variations.
    
    \item \textbf{Scalability to Massive Models:} The method demonstrates exceptional scalability. On the massive LLaMA-3.1-70B, FluxBin maintains the same high-efficiency profile as on smaller models. This confirms that the overhead of our mixed-precision strategy does not grow disproportionately with model size, making it an ideal solution for deploying large-scale foundation models.
\end{itemize}

\begin{figure*}[t]
    \centering
    \includegraphics[width=0.72\linewidth]{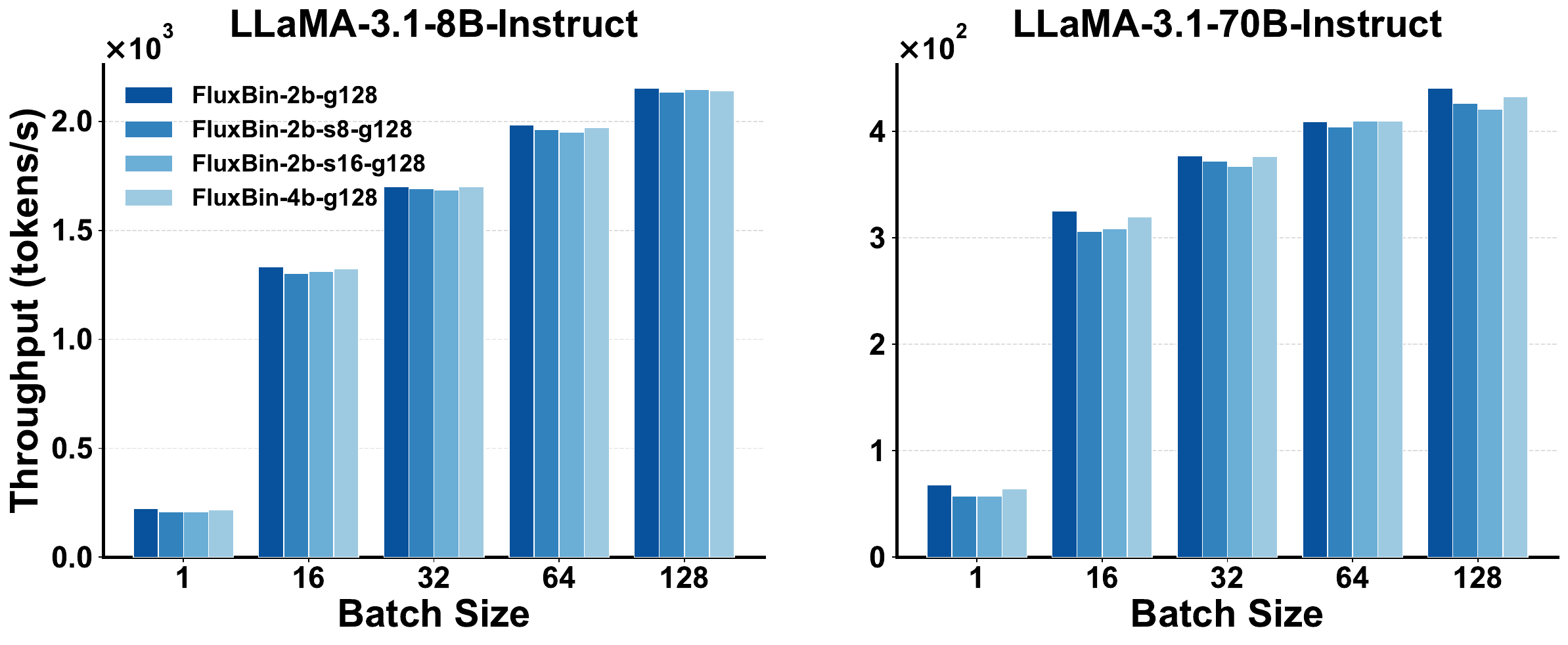}
    \caption{Throughput Comparison on LLaMA-3.1 Across All Batch Sizes}
    \label{fig:throughput-llama3}
\end{figure*}

\begin{figure*}[t]
    \centering
    \includegraphics[width=0.85\linewidth]{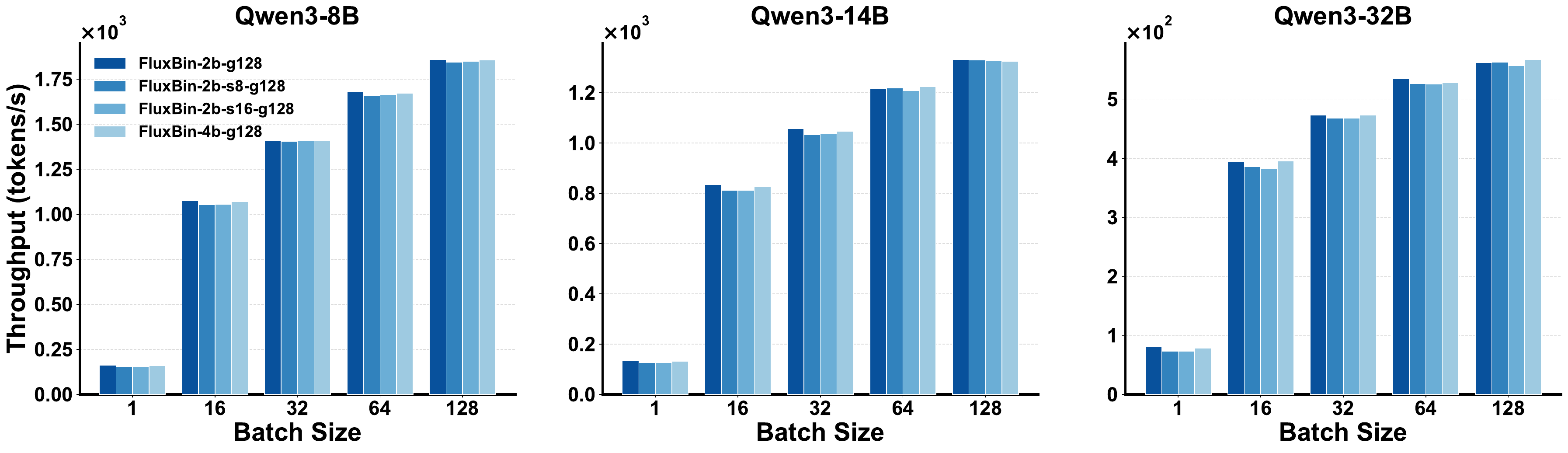}
    \caption{Throughput Comparison on Qwen3 Across All Batch Sizes}
    \label{fig:throughput-qwen3}
\end{figure*}

\subsection{Extended Evaluation on Speed vs. Accuracy}
\label{sec:extension_eval}

To comprehensively validate the generalization capability of FluxBin across diverse and modern LLM architectures, we extend our evaluation to the {LLaMA-3.1} family (8B-Instruct and 70B) and the {Qwen3} family (8B and 14B). The results are summarized in Table~\ref{tab:llama31-task} and Table~\ref{tab:qwen3-task}.

\textbf{Robustness on LLaMA-3.1.} 
LLaMA-3.1 poses significant challenges for low-bit quantization due to its complex activation distributions and high sensitivity. As shown in Table~\ref{tab:llama31-task}, existing 2-bit methods like FlexRound~\cite{FlexRound} and AQLM suffer severe accuracy degradation, particularly on the 70B model where their average accuracy drops to 36.18\% and 50.27\%, respectively. In stark contrast, FluxBin maintains exceptional fidelity. Our {2b-s16-g128} configuration achieves an average accuracy of {67.81\%}, outperforming AQLM by {+17.5\%} and FlexRound by {+31.6\%}. This underscores the effectiveness of our Hessian-guided hybrid base strategy in preserving critical information even in the most sensitive models.
In terms of inference speed, FluxBin continues to lead. On LLaMA-3.1-8B, it reaches {217.03 tokens/s}, surpassing AQLM (124.50 tokens/s) by {1.74$\times$}. Even on the massive 70B model, FluxBin achieves 55.69 tokens/s, significantly faster than AQLM (19.00 tokens/s) and competitive with the higher-precision GPTQ (3-bit).

\textbf{Adaptability to Qwen3 Architectures.}
We further evaluate FluxBin on the Qwen3 family to assess its adaptability to architectures employing Grouped-Query Attention (GQA). As detailed in Table~\ref{tab:qwen3-task}, FluxBin demonstrates strong performance retention. Compared to the SLiM-LLM baseline, FluxBin achieves substantial accuracy gains (e.g., +18.3\% on 8B Avg). More importantly, despite using a lower bit-width ($\sim$2.75 bits), FluxBin's accuracy is highly competitive with the 3.125-bit GPTQ baseline (e.g., 69.39\% vs. 70.45\% on 14B), while offering the theoretical throughput advantages of binary kernels. This confirms that FluxBin's decoupled row-column decomposition and salient-aware refinement are architecture-agnostic, effectively handling modern LLM designs.

\begin{table*}[!htbp]
\centering
\small
\caption{Comparison of quantization methods on LLaMA-3.1 models.
Higher accuracy and speed are better ($\uparrow$). $^{*}$ HS = HellaSwag, WG = Winogrande. 
$^{\dagger}${MMLU:} Evaluated using 5-shot in-context learning. 
We highlight the \colorbox{rank1}{best}, \colorbox{rank2}{second}, and \colorbox{rank3}{third} results (excluding FP16).}
\label{tab:llama31-task}
\setlength{\tabcolsep}{4pt}
\renewcommand{\arraystretch}{1.1}
\resizebox{\textwidth}{!}{\begin{tabular}{l l c c c c c c c c}
\toprule
\textbf{LLaMA-3.1} & \textbf{Method} & \textbf{\#W bit} &
\textbf{Speed (tok/s)} $\uparrow$ &
\textbf{ARC-e} $\uparrow$ &
\textbf{ARC-c} $\uparrow$ &
\textbf{HS}$^{*}$ $\uparrow$ &
\textbf{WG}$^{*}$ $\uparrow$ &
\textbf{MMLU}$^{\dagger}$ $\uparrow$ &
\textbf{Avg.} $\uparrow$ \\
\midrule
\multirow{6}{*}{8B-Instruct}
& FP16                     & 16    & 41.51  & 82.15 & 52.05 & 60.77 & 74.35 & 69.40 & 67.74 \\
& FlexRound                & 2.125 & - & 24.75 & 24.57 & 43.78 & 55.16 & 24.27 & 34.51 \\
& AQLM                     & 2     & \third{124.50} & 46.25 & 30.89 & \best{61.40} & 58.25 & 42.29 & 47.82 \\
& GPTQ                     & 3.125 & --     & \best{72.14} & \best{37.97} & \secondbest{53.73} & \best{68.75} & \best{46.90} & \best{55.90} \\
& FluxBin-2b-s8-g128       & 2.63  & \secondbest{217.02} & \third{69.23} & \third{35.41} & \third{45.16} & \third{68.19} & \third{43.90} & \third{52.38} \\
& FluxBin-2b-s16-g128      & 2.75  & \best{217.03} & \secondbest{69.91} & \secondbest{36.77} & 44.49 & \secondbest{68.43} & \secondbest{45.15} & \secondbest{52.95} \\
\midrule
\multirow{6}{*}{70B}
& FP16                     & 16    & OOM    & 86.66 & 64.93 & 85.03 & 79.64 & 78.58 & 78.97 \\
& FlexRound                & 2.125 & -  & 25.13 & 24.83 & 50.67 & 53.59 & 26.70 & 36.18 \\
& AQLM                     & 2     & \third{19.00}  & 48.82 & 28.67 & 52.83 & 59.59 & 61.45 & 50.27 \\
& GPTQ                     & 3.125 & --     & \third{80.72} & \secondbest{51.79} & \best{63.63} & \best{79.72} & \best{72.87} & \best{69.75} \\
& FluxBin-2b-s8-g128       & 2.63  & \best{58.38}  & \secondbest{81.27} & \third{49.66} & \secondbest{58.06} & \third{79.01} & \third{64.71} & \third{66.54} \\
& FluxBin-2b-s16-g128      & 2.75  & \secondbest{55.69}  & \best{82.73} & \best{53.24} & \third{57.43} & \secondbest{79.48} & \secondbest{66.17} & \secondbest{67.81} \\
\bottomrule
\end{tabular}}
\end{table*}

\begin{table*}[!htbp]
\centering
\small
\caption{Comparison of quantization methods on Qwen3 models.
Higher accuracy and speed are better ($\uparrow$). $^{*}$ HS = HellaSwag, WG = Winogrande. 
We highlight the \colorbox{rank1}{best}, \colorbox{rank2}{second}, and \colorbox{rank3}{third} results (excluding FP16).
}
\label{tab:qwen3-task}
\setlength{\tabcolsep}{4pt}
\renewcommand{\arraystretch}{1.1}
\resizebox{\textwidth}{!}{\begin{tabular}{l l c c c c c c c c c c}
\toprule
\textbf{Qwen3} & \textbf{Method} & \textbf{\#W bit} &
\textbf{Speed (tok/s)} $\uparrow$ &
\textbf{Wiki} $\downarrow$ &
\textbf{PIQA} $\uparrow$ &
\textbf{ARC-e} $\uparrow$ &
\textbf{ARC-c} $\uparrow$ &
\textbf{BoolQ} $\uparrow$ &
\textbf{HS}$^{*}$ $\uparrow$ &
\textbf{WG}$^{*}$ $\uparrow$ &
\textbf{Avg.} $\uparrow$ \\
\midrule
\multirow{5}{*}{8B}
& FP16                     & 16    & 28.60  & 9.72 & 76.39 & 83.29 & 55.38 & 86.57 & 57.16 & 68.03 & 71.14 \\
& SLiM-LLM                 & 2     & -      & 30.61 & 60.11 & 43.81 & 23.97 & 62.99 & 34.83 & 55.80 & 46.92 \\
& GPTQ                     & 3.125 & -      & \best{11.18} & \third{74.16} & \best{77.90} & \best{48.63} & \best{84.07} & \best{53.42} & \third{65.43} & \best{67.27} \\
& FluxBin-2b-s8-g128       & 2.63  & \secondbest{121.85} & \third{13.46} & \best{74.37} & \third{74.41} & \third{42.58} & \secondbest{83.33} & \third{47.90} & \secondbest{66.30} & \third{64.82} \\
& FluxBin-2b-s16-g128      & 2.75  & \best{121.97} & \secondbest{12.44} & \secondbest{74.21} & \secondbest{74.71} & \secondbest{43.77} & \third{82.75} & \secondbest{49.09} & \best{66.77} & \secondbest{65.22} \\
\midrule
\multirow{5}{*}{14B}
& FP16                     & 16    & 26.54  & 8.64 & 80.36 & 84.13 & 58.62 & 89.42 & 60.94 & 72.69 & 74.36 \\
& Slim-LLm                 & 2     & -      & 22.85 & 61.83 & 52.54 & 29.35 & 61.20 & 31.52 & 52.04 & 48.08 \\
& GPTQ                     & 3.125 & -      & \best{9.73} & \best{77.42} & \best{79.12} & \best{52.82} & \third{86.64} & \best{58.02} & \third{68.67} & \best{70.45} \\
& FluxBin-2b-s8-g128       & 2.63  & \secondbest{102.07} & \third{11.60} & \third{74.59} & \third{77.74} & \secondbest{51.96} & \best{87.43} & \secondbest{53.96} & \secondbest{69.30} & \third{69.16} \\
& FluxBin-2b-s16-g128      & 2.75  & \best{102.09} & \secondbest{11.42} & \secondbest{76.17} & \secondbest{78.66} & \third{51.62} & \secondbest{86.70} & \third{53.74} & \best{69.46} & \secondbest{69.39} \\
\bottomrule
\end{tabular}}
\end{table*}

\subsection{Impact of Group Size}
\label{sec:ablation_groupsize}

We investigate the impact of group size (g) on the trade-off between model accuracy and effective bit-width. As detailed in Table~\ref{tab:ablation-groupsize}, we evaluate three granularities (g=\{64, 128, 256\}) across pure and hybrid configurations.

\textbf{Sensitivity of Global Base.} The pure binary configuration ({2b}) exhibits high sensitivity to group granularity. Increasing $g$ from 64 to 256 results in a severe degradation in average accuracy, dropping from 57.96\% to 47.20\%. This indicates that a coarse-grained global binary approximation struggles to capture the local weight variance without refinement.

\textbf{Robustness of Hybrid Bases.} In contrast, our salient-aware hybrid configurations ({2b-s8} and {2b-s16}) demonstrate remarkable robustness to larger group sizes. For instance, {2b-s16} maintains a high average accuracy of 58.87\% even at g=256, significantly outperforming the pure {2b} baseline at g=64 (57.96\%). This confirms that our salient refinement effectively compensates for the information loss caused by coarse-grained grouping.

\textbf{Optimal Configuration.} While g=64 yields the highest accuracy, it inflates the effective bit-width to $\sim$3.5 bits (for s16). Conversely, g=128 strikes the optimal balance, achieving competitive accuracy (59.15\%) with a compact footprint ($\sim$2.75 bits), justifying its selection as our default setting.
\begin{table*}[t]
\centering
\caption{Ablation study across different group sizes (g) in fixed configurations of LLaMA-2-7B. Lower Wiki PPL is better, while higher accuracy is better for downstream tasks. $^{\dagger}$ The average bit includes weight and scale factors. The calculation for FluxBin is derived in Sec.~\ref{sec:appendix-storage}. 
$^{\ddagger}$ HS = HellaSwag, WG = Winogrande.}
\label{tab:ablation-groupsize}
\setlength{\tabcolsep}{6pt}
\renewcommand{\arraystretch}{1.1}
\small
\resizebox{\textwidth}{!}{\begin{tabular}{c c c c c c c c c c}
\toprule
\textbf{Config.} & \textbf{W bit}$^{\dagger}$ &
\textbf{Wiki} $\downarrow$ & \textbf{PIQA} $\uparrow$ & \textbf{ARC-e} $\uparrow$ &
\textbf{ARC-c} $\uparrow$ & \textbf{BoolQ} $\uparrow$ &
\textbf{HS}$^{\ddagger}$ $\uparrow$ & \textbf{WG}$^{\ddagger}$ $\uparrow$ & \textbf{Avg.} $\uparrow$\\
\midrule
2b-g64   & 2.51 & 10.27 & 73.50 & 63.43 & 30.29 & 67.00 & 48.57 & 64.96 & 57.96 \\
2b-g128  & 2.25 & 13.71 & 69.79 & 56.23 & 25.00 & 65.41 & 38.64 & 56.69 & 51.96 \\
2b-g256  & 2.13 & 12.79 & 64.91 & 48.91 & 22.35 & 62.08 & 34.52 & 50.43 & 47.20 \\
\midrule
2b-s8-g64   & 3.25 & 7.95 & 73.01 & 67.59 & 32.51 & 73.70 & 47.05 & 66.06 & 59.99 \\
2b-s8-g128  & 2.63 & 8.76 & 72.52 & 65.45 & 30.97 & 69.88 & 45.64 & 64.09 & 58.09 \\
2b-s8-g256  & 2.31 & 8.91 & 72.47 & 64.98 & 30.46 & 71.25 & 44.97 & 64.88 & 58.17 \\
\midrule
2b-s16-g64   & 3.50 & 7.40 & 74.76 & 70.54 & 34.22 & 72.72 & 48.26 & 64.96 & 60.91 \\
2b-s16-g128  & 2.75 & 8.20 & 75.50 & 65.40 & 32.42 & 70.80 & 45.97 & 64.80 & 59.15 \\
2b-s16-g256  & 2.38 & 8.64 & 73.34 & 65.74 & 31.48 & 72.66 & 46.06 & 63.93 & 58.87 \\
\bottomrule
\end{tabular}}
\end{table*}

\subsection{Sensitivity to Calibration Sample Size}
\label{sec:ablation_nsample}

We investigate the impact of the calibration c4 dataset size ($N_{cal}$) on both quantization performance and computational overhead for LLaMA-2-7B FluxBin-2b-s8-g128. As shown in Table~\ref{tab:nsample-ablation}, we vary $N_{cal}$ from 64 to 1024. The results reveal a distinct "rise-then-fall" trend in model performance.

\textbf{Hessian Estimation vs. Overfitting.}
The performance trajectory can be explained through the approximation of the Hessian matrix $\mathbf{H} \approx \mathbb{E}[\mathbf{x}\mathbf{x}^T]$, which guides the quantization objective $\Delta \mathcal{L} \approx \frac{1}{2} \Delta \mathbf{w}^T \mathbf{H} \Delta \mathbf{w}$.
\begin{itemize}
    \item \textbf{Under-sampling ($N_{cal} < 256$):} When samples are scarce, the Hessian estimate only captures the principal directions of activation magnitudes. Increasing $N_{cal}$ from 64 to 256 significantly improves the average accuracy (57.86\% $\to$ 58.77\%) and reduces Wiki PPL (8.92 $\to$ 8.63). This gain stems from a more stable and representative estimation of $\mathbf{H}$, allowing the algorithm to better identify sensitive weights.
    \item \textbf{Calibration Overfitting ($N_{cal} > 512$):} Surprisingly, further increasing $N_{cal}$ to 1024 leads to performance degradation (Avg. Acc drops to 58.07\%). This phenomenon is {calibration overfitting}, occurs when the Hessian captures fine-grained distribution details specific to the calibration set that do not generalize to the test distribution. The quantization decision boundary becomes overly biased towards the calibration data, harming generalization.
\end{itemize}

\textbf{Efficiency Trade-off.}
In terms of cost, the PTQ time increases linearly with $N_{cal}$. While $N_{cal}=1024$ incurs a lengthy calibration time ($\sim$75 min), $N_{cal}=256$ offers the optimal trade-off, achieving peak accuracy with a moderate time cost ($\sim$25 min). Consequently, we adopt $N_{cal}=256$ as the default setting for our main experiments.
\begin{table*}[t]
\centering
\caption{Ablation study of the number of calibration samples ({nsample}) on evaluation time and performance on LLaMA-2-7B, with FluxBin configuration 2b-s8-g128. Lower PTQ time and Wiki PPL are better ($\downarrow$), while higher accuracy is better ($\uparrow$). $^{\ddagger}$ HS = HellaSwag, WG = Winogrande.}
\label{tab:nsample-ablation}
\setlength{\tabcolsep}{4pt}
\renewcommand{\arraystretch}{1.1}
\resizebox{\textwidth}{!}{\begin{tabular}{c c c c c c c c c c}
\toprule
\textbf{nsample} & \textbf{Time (min)} $\downarrow$ & \textbf{Wiki} $\downarrow$ & \textbf{PIQA} $\uparrow$ & \textbf{ARC-e} $\uparrow$ &
\textbf{ARC-c} $\uparrow$ & \textbf{BoolQ} $\uparrow$ &
\textbf{HS}$^{\ddagger}$ $\uparrow$ & \textbf{WG}$^{\ddagger}$ $\uparrow$ & \textbf{Avg.} $\uparrow$ \\
\midrule
64   & 13.22 & 8.92 & 72.31 & 64.60 & 30.55 & 70.06 & 45.61 & 64.01 & 57.86 \\
128  & 17.02 & 8.85 & 72.74 & 65.49 & 30.97 & 69.88 & 45.38 & 63.54 & 58.00 \\
256  & 25.53 & 8.63 & 74.10 & 64.44 & 33.28 & 71.13 & 45.51 & 64.17 & 58.77 \\
512  & 41.78 & 8.54 & 73.34 & 65.78 & 32.00 & 71.13 & 45.89 & 63.30 & 58.57 \\
1024 & 74.58 & 8.76 & 72.63 & 65.40 & 30.80 & 69.82 & 45.67 & 64.09 & 58.07 \\
\bottomrule
\end{tabular}}
\end{table*}

\section{Profiling Study}
\subsection{Energy Efficiency and Resource Utilization}
\label{sec:energy_utilization}

To evaluate the environmental impact and deployment cost of FluxBin, we profile power consumption and hardware utilization on an NVIDIA A100 GPU using \texttt{nvidia-smi} at 100ms sampling intervals during steady-state decode generation (excluding the first 10s of warmup). Total energy is computed as $E = \sum_t P(t) \cdot \Delta t$. Table~\ref{tab:GPU-profile} summarizes the results across the LLaMA-2 family, with input and output length 128, batch size 1.

\textbf{Energy Consumption.} 
FluxBin demonstrates a dramatic reduction in total energy cost (Joules). For LLaMA-2-13B, FluxBin reduces total energy consumption by over {10.19$\times$} compared to the FP16 baseline (579.23J vs. 5908.03J). This massive efficiency gain stems from two factors: \textbf{(i)} the significant reduction in inference latency (as discussed in Sec.~\ref{sec:speed_acc}), which reduces the total active time of the GPU; and \textbf{(ii)} lower average power draw ($\sim$115W vs. 225W). The reduced power draw is attributed to the minimized data movement overhead—since memory access energy typically dominates arithmetic energy—and the use of lightweight bitwise/integer operations instead of power-hungry FP16 arithmetic.

\textbf{Hardware Utilization.}
A notable observation is the {GPU busy time} reported by \texttt{nvidia-smi}, where FluxBin consistently stays near $>$99.8\%, whereas the FP16 baseline hovers around 67\%--86\%. We stress that this is a device-busy fraction, not a measure of arithmetic-unit saturation; read together with the very low memory utilization below, it indicates that the LUT kernel keeps the SMs continuously occupied and is not stalled waiting on HBM traffic, unlike the FP16 baseline.
Simultaneously, the {memory-bandwidth utilization} remains low (e.g., $\sim$6\% for 7B). This is not a performance target in itself; it reflects the small weight traffic of the compressed 2-bit format and indicates that the kernel is limited by on-chip LUT/compute throughput rather than HBM bandwidth. For LLaMA-2-70B, where FP16 fails due to OOM, FluxBin operates efficiently with high GPU utilization and stable power consumption ($\sim$155W), further validating its scalability for massive model deployment.
\begin{table*}[!htbp]
\centering
\small
\caption{Power, energy, and resource utilization comparison on LLaMA-2 models.
Lower power and energy are better ($\downarrow$); GPU and memory utilization are reported as diagnostics of the compute/bandwidth balance rather than as quality targets.}
\label{tab:GPU-profile}
\setlength{\tabcolsep}{6pt}
\renewcommand{\arraystretch}{1.1}
\resizebox{\textwidth}{!}{\begin{tabular}{l l c c c c}
\toprule
\textbf{Model} & \textbf{Quantization} & \textbf{Avg Power (W)} $\downarrow$ &
\textbf{Total Energy (J)} $\downarrow$ &
\textbf{Avg GPU Util. (\%)} &
\textbf{Avg Mem. Util. (\%)} \\
\midrule
\multirow{5}{*}{LLaMA-2-7B}
& FP16                & 168.66 & 3467.69           & 67.01 & 18.09 \\
& FluxBin-2b-g128     & 106.75 & 462.98  & 99.83 & 5.55  \\
& FluxBin-2b-s8-g128  & 110.80 & 425.05   & 99.81 & 6.27  \\
& FluxBin-2b-s16-g128 & 107.40 & 434.34   & 99.87 & 6.27  \\
& FluxBin-4b-g128     & 109.40 & 401.16   & 99.85 & 7.65  \\
\midrule
\multirow{5}{*}{LLaMA-2-13B}
& FP16                & 225.49 & 5908.03           & 86.83 & 33.16 \\
& FluxBin-2b-g128     & 113.86 & 579.23   & 99.89 & 8.15  \\
& FluxBin-2b-s8-g128  & 115.82 & 603.46    & 99.98 & 9.69  \\
& FluxBin-2b-s16-g128 & 115.07 & 644.74    & 99.93 & 9.69  \\
& FluxBin-4b-g128     & 112.57 & 560.84   & 99.89 & 12.45 \\
\midrule
\multirow{5}{*}{LLaMA-2-70B}
& FP16                & OOM    & OOM                & OOM   & OOM   \\
& FluxBin-2b-g128     & 150.22 & 2522.47            & 99.95 & 26.54 \\
& FluxBin-2b-s8-g128  & 157.08 & 2441.06            & 99.98 & 34.19 \\
& FluxBin-2b-s16-g128 & 154.56 & 2560.80            & 99.98 & 34.19 \\
& FluxBin-4b-g128     & 159.68 & 2477.17            & 99.96 & 49.14 \\
\bottomrule
\end{tabular}}
\end{table*}

\subsection{Shared Memory Footprint Analysis}
\label{sec:shared_mem_analysis}

Efficient utilization of Shared Memory (SRAM) is critical for LUT-based kernels to maximize occupancy and hide memory latency. We analyze the shared memory consumption for a single thread block based on our default configuration ($M_{\text{tile}}=128, K_{\text{tile}}=64$).

The total shared memory usage $S_{\text{total}}$ consists of two components: the storage for Look-Up Tables ($S_{\text{LUT}}$) and the intermediate accumulation buffer ($S_{\text{acc}}$).

\textbf{LUT Storage ($S_{\text{LUT}}$).} 
Since the input vector $\mathbf{X}$ varies across the batch, and the column scaling factor $\boldsymbol{\alpha}_c$ is fused into the LUT, independent LUTs must be constructed for each batch index within the block. For a sub-vector size $\mu=8$ and $K_b$ binary bases, the LUT storage is calculated as:
\begin{equation}
S_{\text{LUT}} = B_{\text{tile}} \cdot K_b \cdot \frac{K_{\text{tile}}}{\mu} \cdot 2^\mu \cdot \text{sizeof(float)}
\end{equation}
where $B_{\text{tile}}$ denotes the batch size handled per block. The term $\frac{K_{\text{tile}}}{\mu}$ represents the number of sub-vectors per tile.

\textbf{Accumulation Buffer ($S_{\text{acc}}$).} 
To perform parallel reduction and row-scale post-correction, we maintain a buffer to store partial sums for each row and each binary base before the final reduction.
\begin{equation}
S_{\text{acc}} = B_{\text{tile}} \cdot K_b \cdot M_{\text{tile}} \cdot \text{sizeof(float)}
\end{equation}

\textbf{Impact of Sub-vector Size ($\mu$) on Resource Trade-offs.}
Table~\ref{tab:mu-profile} presents a comparative analysis of shared memory footprint, throughput, and energy efficiency between sub-vector sizes $\mu=4$ and $\mu=8$.

\textbf{(i)}{ Shared Memory Occupancy:} 
Consistent with our theoretical derivation, $\mu=4$ maintains an extremely low footprint ($<$24 KB), imposing negligible pressure on SRAM. In contrast, $\mu=8$ exhibits significantly higher consumption due to the exponential growth of LUT size ($2^8$ vs $2^4$). Specifically, at Batch Size (BS) 8, $\mu=8$ consumes 136 KB. Crucially, this remains within the 164 KB capacity limit per SM on the NVIDIA A100, ensuring that the larger LUT configuration does not cause shared memory overflow or prevent kernel execution.
\begin{table*}[t]
\centering
\caption{Impact of sub-vectors size $\mu$ on block-level shared memory usage and end-to-end generation performance of FluxBin-2b-g128 LLaMA-2-7B.
Higher throughput is better ($\uparrow$), while lower power and energy are better ($\downarrow$).}
\label{tab:mu-profile}
\setlength{\tabcolsep}{6pt}
\renewcommand{\arraystretch}{1.1}
\begin{tabular}{c c c c c c}
\toprule
\textbf{$\mu$} & \textbf{BS} & \textbf{Shared Mem. (KB)} &
\textbf{Throughput} $\uparrow$ &
\textbf{Power (W)} $\downarrow$ &
\textbf{Energy (J)} $\downarrow$ \\
\midrule
\multirow{3}{*}{4}
 & 1 & 3   & 213.97 & 107.13 & 397.15 \\
 & 4 & 12  & 656.50 & 146.97 & 802.74 \\
 & 8 & 24  & 952.86 & 172.31 & 1301.12 \\
\midrule
\multirow{3}{*}{8}
 & 1 & 17  & 213.96 & 105.41 & 451.18 \\
 & 4 & 68  & 655.58 & 140.04 & 792.82 \\
 & 8 & 136 & 953.04 & 169.48 & 1220.45 \\
\bottomrule
\end{tabular}
\end{table*}

\textbf{(ii)}{ Performance and Selection Strategy:} 
The end-to-end throughput remains comparable between $\mu=4$ and $\mu=8$, suggesting it is not a shared memory bound regime. However, we select $\mu=8$ as the default configuration driven by {energy efficiency}.
According to our complexity analysis (Eq.~\ref{equa:complexity-full}), the number of sub-vectors within a fixed $K$-tile is inversely proportional to $\mu$ (sub-vectors number = ${K_{\text{tile}}}/{\mu}$). Consequently, decreasing $\mu$ from 8 to 4 {doubles} the number of sub-vectors. This necessitates a $2\times$ increase in the frequency of LUT read operations, which is shown to dominate the complexity in Sec.~\ref{sec:appendix-complexity}, and the associated accumulation within the block. As empirically evidenced in Table~\ref{tab:mu-profile}, this intensified operational frequency directly translates to higher power consumption (e.g., an increase of $\sim$3W at BS=8), making $\mu=8$ the superior choice for energy-efficient inference.

\subsection{Tiling Strategy and TFLOPS Comparison}
\label{sec:kernel-tile}
To maximize the utilization of GPU compute units and memory bandwidth, we conduct a sensitivity analysis on the matrix tiling parameters. Specifically, we investigate the tile size along the output channel dimension ($M_{\text{tile}}$), the input channel dimension ($K_{\text{tile}}$), and the salient refinement block size ($M_{s\_tile}$).

\textbf{Performance Metric.} 
We evaluate the kernel throughput using {Effective TFLOPS} (Tera Floating-point Operations Per Second). Since FluxBin utilizes integer-based LUT logic rather than floating-point arithmetic, this metric represents the equivalent computational throughput required to perform the same matrix-vector multiplication in FP16.
For hybrid configurations, we explicitly account for the additional workload introduced by the salient refinement branch. The Effective TFLOPS is calculated as:
\begin{equation}
\text{TFLOPS} = \frac{2 \cdot M \cdot (N + N_{sal})}{\text{Latency (s)} \times 10^{12}}
\label{eq:effective_tflops}
\end{equation}
where $M$ and $N$ denote the output and input channel dimensions, respectively, and $N_{sal}$ is the number of salient columns. For standard global base kernels (e.g., 2b-g128), we simply set $N_{sal} = 0$. %

\textbf{Optimization of Global Base Tiling.} 
Table~\ref{tab:tile-dense} presents the performance of pure binary kernels (2-bit and 4-bit) under various $(M_{\text{tile}}, K_{\text{tile}})$ combinations.
\begin{itemize}
    \item \textbf{Impact of $M_{\text{tile}}$:} Increasing $M_{\text{tile}}$ yields the most significant performance gain. For instance, with the 2-bit kernel at $K=4096$, expanding $M_{\text{tile}}$ from 32 to 128 increases throughput from 0.30 TFLOPS to 0.81 TFLOPS. A larger $M_{\text{tile}}$ increases the number of threads per block, effectively hiding instruction latency and amortizing the overhead of LUT construction.
    \item \textbf{Impact of $K_{\text{tile}}$:} The choice of $K_{\text{tile}}$ involves a trade-off between reduction efficiency and shared memory occupancy. While larger $K_{\text{tile}}$ reduces the number of global atomic reductions, it consumes more shared memory for LUT storage. Our empirical results suggest that $K_{\text{tile}}=32$ or $64$ typically offers the optimal balance.
\end{itemize}

\textbf{Optimization of Mixed-Precision Tiling.}
Table~\ref{tab:tile-saliency} extends the analysis to hybrid bases configurations (2b-s8/s16-g128). We observe that the mixed-precision kernel exhibits tiling sensitivities similar to the global base kernel, with $M_{\text{tile}}=128$ consistently delivering the best performance. Regarding the salient-specific tile size $M_{s\_tile}$, setting it to 64 provides a slight latency reduction compared to 16, as it promotes better memory coalescing for the dense salient matrix.
Notably, the performance gap between {s8} and {s16} configurations is marginal across all tiling setups. This confirms the efficiency of our Virtual Columnar Mapping (VCM) strategy, which ensures that scaling up the salient branch incurs minimal computational penalties.

Based on these findings, we adopt $(M_{\text{tile}}=128, K_{\text{tile}}=64, M_{s\_tile}=64)$ as default for all end-to-end evaluations.

\begin{table*}[!htbp]
\centering
\caption{Effect of tile sizes on FluxBin kernel performance.}
\label{tab:tile-dense}
\setlength{\tabcolsep}{8pt} 
\renewcommand{\arraystretch}{1.2}
\sisetup{table-format=3.1} 

\resizebox{\textwidth}{!}{\begin{tabular}{c c c c S S[table-format=1.2] S S[table-format=1.2]}
\toprule
\multicolumn{4}{c}{} &
\multicolumn{2}{c}{\textbf{FluxBin-2b-g128}} &
\multicolumn{2}{c}{\textbf{FluxBin-4b-g128}} \\
\cmidrule(lr){5-6} \cmidrule(lr){7-8}
\textbf{M} & \textbf{K} & \boldmath$M_{\text{tile}}$ & \boldmath$K_{\text{tile}}$ &
{\textbf{Latency (\textmu s)} $\downarrow$} & {\textbf{TFLOPS} $\uparrow$} &
{\textbf{Latency (\textmu s)} $\downarrow$} & {\textbf{TFLOPS} $\uparrow$} \\
\midrule
\multirow{9}{*}{4096} & \multirow{9}{*}{4096} 
  & 32  & 32  & 113.3 & 0.30 & 109.6 & 0.31 \\
& & 32  & 64  & 143.0 & 0.23 & 108.4 & 0.31 \\
& & 32  & 128 & 228.9 & 0.15 & 108.4 & 0.31 \\
\cmidrule{3-8}
& & 64  & 32  & 58.0  & 0.58 & 72.9  & 0.46 \\
& & 64  & 64  & 67.6  & 0.50 & 66.9  & 0.50 \\
& & 64  & 128 & 84.6  & 0.40 & 66.8  & 0.50 \\
\cmidrule{3-8}
& & 128 & 32  & 41.2  & \textbf{0.81} & 57.8  & 0.58 \\
& & 128 & 64  & 41.4  & \textbf{0.81} & 56.3  & \textbf{0.60} \\
& & 128 & 128 & 47.5  & 0.71 & 56.4  & \textbf{0.60} \\
\midrule 
\multirow{9}{*}{4096} & \multirow{9}{*}{11008} 
  & 32  & 32  & 289.2 & 0.31 & 296.4 & 0.30 \\
& & 32  & 64  & 379.9 & 0.24 & 307.0 & 0.29 \\
& & 32  & 128 & 594.5 & 0.15 & 308.7 & 0.29 \\
\cmidrule{3-8}
& & 64  & 32  & 148.5 & 0.61 & 215.1 & 0.42 \\
& & 64  & 64  & 166.3 & 0.54 & 200.9 & 0.45 \\
& & 64  & 128 & 208.8 & 0.43 & 201.0 & 0.45 \\
\cmidrule{3-8}
& & 128 & 32  & 102.8 & \textbf{0.88} & 158.9 & \textbf{0.57} \\
& & 128 & 64  & 103.5 & 0.87 & 181.3 & 0.50 \\
& & 128 & 128 & 114.8 & 0.79 & 182.1 & 0.50 \\
\bottomrule
\end{tabular}}
\end{table*}

\begin{table*}[!htbp]
\centering
\caption{Effect of tile sizes on FluxBin kernel performance on hybrid bases.}
\label{tab:tile-saliency}
\setlength{\tabcolsep}{5pt}
\renewcommand{\arraystretch}{1.15}
\sisetup{table-format=3.1}

\resizebox{\textwidth}{!}{\begin{tabular}{c c c c c S S[table-format=1.2] S S[table-format=1.2]}
\toprule
\multicolumn{5}{c}{} &
\multicolumn{2}{c}{\textbf{FluxBin-2b-s8-g128}} &
\multicolumn{2}{c}{\textbf{FluxBin-2b-s16-g128}} \\
\cmidrule(lr){6-7} \cmidrule(lr){8-9}
\textbf{M} & \textbf{K} & \boldmath$K_{\text{tile}}$ & \boldmath$M_{\text{tile}}$ & \boldmath$M_{s\_tile}$ &
{\textbf{Lat (\textmu s)} $\downarrow$} & {\textbf{TFLOPS} $\uparrow$} &
{\textbf{Lat (\textmu s)} $\downarrow$} & {\textbf{TFLOPS} $\uparrow$} \\
\midrule

\multirow{18}{*}{4096} & \multirow{18}{*}{4096} &
\multirow{9}{*}{32} & \multirow{3}{*}{32} & 16 & 121.0 & 0.29 & 124.4 & 0.28 \\
& & & & 32 & 119.5 & 0.29 & 123.2 & 0.28 \\
& & & & 64 & 118.3 & 0.29 & 121.9 & 0.28 \\
\cmidrule{4-9}
& & & \multirow{3}{*}{64} & 16 & 67.3 & 0.51 & 69.0 & 0.50 \\
& & & & 32 & 65.7 & 0.53 & 67.9 & 0.51 \\
& & & & 64 & 64.7 & 0.53 & 66.5 & 0.52 \\
\cmidrule{4-9}
& & & \multirow{3}{*}{128} & 16 & 51.5 & 0.67 & 53.5 & 0.65 \\
& & & & 32 & 50.0 & 0.69 & 52.1 & 0.66 \\
& & & & 64 & 48.6 & \textbf{0.71} & 50.9 & \textbf{0.68} \\
\cmidrule{3-9} 
& & \multirow{9}{*}{64} & \multirow{3}{*}{32} & 16 & 157.7 & 0.22 & 158.7 & 0.22 \\
& & & & 32 & 154.4 & 0.22 & 156.5 & 0.22 \\
& & & & 64 & 155.1 & 0.22 & 156.6 & 0.22 \\
\cmidrule{4-9}
& & & \multirow{3}{*}{64} & 16 & 76.3 & 0.45 & 78.2 & 0.44 \\
& & & & 32 & 74.6 & 0.46 & 77.0 & 0.45 \\
& & & & 64 & 73.7 & 0.47 & 75.6 & 0.46 \\
\cmidrule{4-9}
& & & \multirow{3}{*}{128} & 16 & 49.8 & 0.69 & 51.7 & 0.67 \\
& & & & 32 & 48.3 & 0.72 & 50.5 & 0.69 \\
& & & & 64 & 47.1 & \textbf{0.74} & 49.2 & \textbf{0.70} \\
\midrule

\multirow{18}{*}{4096} & \multirow{18}{*}{11008} &
\multirow{9}{*}{32} & \multirow{3}{*}{32} & 16 & 297.8 & 0.31 & 304.2 & 0.30 \\
& & & & 32 & 296.5 & 0.31 & 303.1 & 0.30 \\
& & & & 64 & 295.3 & 0.31 & 301.3 & 0.30 \\
\cmidrule{4-9}
& & & \multirow{3}{*}{64} & 16 & 159.4 & 0.57 & 161.1 & 0.57 \\
& & & & 32 & 157.9 & 0.58 & 159.8 & 0.57 \\
& & & & 64 & 156.8 & 0.58 & 157.9 & 0.58 \\
\cmidrule{4-9}
& & & \multirow{3}{*}{128} & 16 & 122.4 & 0.75 & 124.1 & 0.73 \\
& & & & 32 & 120.9 & 0.75 & 122.6 & 0.74 \\
& & & & 64 & 119.9 & \textbf{0.76} & 121.4 & \textbf{0.75} \\
\cmidrule{3-9}
& & \multirow{9}{*}{64} & \multirow{3}{*}{32} & 16 & 388.1 & 0.24 & 402.0 & 0.23 \\
& & & & 32 & 386.5 & 0.24 & 405.1 & 0.23 \\
& & & & 64 & 385.2 & 0.24 & 401.5 & 0.23 \\
\cmidrule{4-9}
& & & \multirow{3}{*}{64} & 16 & 172.6 & 0.53 & 174.8 & 0.52 \\
& & & & 32 & 171.1 & 0.53 & 173.4 & 0.53 \\
& & & & 64 & 170.1 & 0.54 & 172.2 & 0.53 \\
\cmidrule{4-9}
& & & \multirow{3}{*}{128} & 16 & 112.1 & 0.81 & 114.1 & 0.80 \\
& & & & 32 & 110.7 & 0.82 & 112.6 & 0.81 \\
& & & & 64 & 109.7 & \textbf{0.83} & 111.4 & \textbf{0.82} \\
\bottomrule
\end{tabular}}
\end{table*}

\subsection{Kernel-Level Latency Comparison}
\label{sec:kernel-latency}
To isolate the kernel contribution from the end-to-end pipeline, we further compare the per-operator GEMV latency of FluxBin against the leading 2-bit weight-only baselines, AQLM and QuIP\#, across the weight shapes that dominate LLM linear layers. Each kernel is timed in isolation using its official implementation. As reported in Table~\ref{tab:kernel-latency}, FluxBin's LUT-based kernel is consistently the fastest at the pure 2-bit setting, e.g. 41\,$\mu$s at $4096\times4096$ versus 112\,$\mu$s for QuIP\# and 79\,$\mu$s for AQLM. This gap reflects the reconstruction cost of the competing schemes: QuIP\# decodes an E8 lattice and AQLM performs a multi-codebook gather per weight, whereas FluxBin reconstructs each weight with a single foldable table lookup plus addition. Crucially, the hybrid-precision configurations (2b-s8/s16-g128) add only a marginal overhead over the pure 2-bit kernel, confirming that the Virtual Columnar Mapping keeps the salient branch nearly free; and even the 4-bit kernel remains faster than the 2-bit baselines at most shapes.

\begin{table*}[!htbp]
    \centering
    \small
    \caption{Per-operator GEMV kernel latency ($\mu$s) across weight matrix sizes ($M\times N$), measured for the 2-bit quantized baselines and FluxBin. Lower is better ($\downarrow$); the best result in each row is in \textbf{bold}. Each kernel is measured in isolation with its official implementation.}
    \label{tab:kernel-latency}
    \setlength{\tabcolsep}{6pt}
    \renewcommand{\arraystretch}{1.15}
    \resizebox{\textwidth}{!}{\begin{tabular}{c cc cccc}
    \toprule
        \multirow{2}{*}{\textbf{Weight Size} ($M\times N$)}
        & \multicolumn{2}{c}{\textbf{Baselines}}
        & \multicolumn{4}{c}{\textbf{FluxBin}} \\
    \cmidrule(lr){2-3} \cmidrule(lr){4-7}
        & \textbf{AQLM 2bit}
        & \textbf{QuIP\# 2bit}
        & \textbf{2b-g128}
        & \textbf{2b-s8-g128}
        & \textbf{2b-s16-g128}
        & \textbf{4b-g128} \\
    \midrule
        4096$\times$4096   & 79  & 112 & \textbf{41}  & 47  & 49  & 56  \\
        5120$\times$5120   & 111 & 174 & \textbf{63}  & 73  & 74  & 87  \\
        4096$\times$11008  & 169 & 272 & \textbf{102} & 109 & 111 & 158 \\
        5120$\times$13824  & 277 & 411 & 167 & \textbf{166} & 169 & 252 \\
    \bottomrule
    \end{tabular}}
\end{table*}

\end{document}